%% file: main.tex
\documentclass[twoside]{article}
\usepackage{minitoc}

\usepackage[accepted]{aistats2026}
\usepackage{mathrsfs,amsmath,amsfonts,amssymb,amstext,amscd, dsfont,pifont,
            amsthm,euscript,color,xcolor,accents,xr} 
\usepackage{algorithmic}
\usepackage{algorithm}
\usepackage{url}
\usepackage{mathtools}
\usepackage{epsfig}
\usepackage{float}
\usepackage{appendix}
\usepackage{floatflt}
\usepackage{nicefrac}
\usepackage{anysize}
\usepackage{enumerate}
\usepackage{epsfig}
\usepackage{graphicx}
\usepackage{mathtools}
\usepackage{upgreek}
\usepackage{pdfpages}
\usepackage{subcaption}  
\usepackage{bbold}
\usepackage{bbm,bm}

\usepackage[linecolor=magenta!60!, backgroundcolor=magenta!10!,textwidth=1.6cm, textcolor=magenta]{todonotes}

\numberwithin{equation}{section} 

\newcommand*{\Scale}[2][4]{\scalebox{#1}{$#2$}}%
\def\nU{n_{\Scale[0.4]{\mathcal{U}}}}
\def\nX{n_{\Scale[0.4]{\mathcal{X}}}}

\input{macros_nicole}

\definecolor{citecolor}{rgb}{0, 0.3, 0.6}

\usepackage[colorlinks=true,linkcolor=citecolor, citecolor=citecolor, backref]{hyperref}
\usepackage[round]{natbib}
\begin{document}

\twocolumn[

\aistatstitle{Random Features for Operator-Valued Kernels: Bridging Kernel Methods and Neural Operators}

\aistatsauthor{ Mike Nguyen \And Nicole M\"ucke }

\aistatsaddress{ Technical University  of Braunschweig \And  Technical University  of Braunschweig  } ]

\begin{abstract}
In this work, we investigate the generalization properties of random feature methods. Our analysis extends prior results for Tikhonov regularization to a broad class of spectral regularization techniques and further generalizes the setting to operator-valued kernels. This unified framework enables a rigorous theoretical analysis of neural operators and neural networks through the lens of the Neural Tangent Kernel (NTK). In particular, it allows us to establish optimal learning rates and provides a good understanding of how many neurons are required to achieve a given accuracy. Furthermore, we establish minimax rates in the well-specified case and also in the misspecified case, where the target is not contained in the reproducing kernel Hilbert space. These results sharpen and complete earlier findings for specific kernel algorithms.
\end{abstract}


\section{Introduction}
Operator learning has become a powerful paradigm in machine learning. 
It is particularly well suited for surrogate modeling in areas such as uncertainty quantification, inverse problems, and design optimization. The central objective in these applications is to approximate potentially nonlinear operators, for instance solution operators of partial differential equations (PDEs), see \citet{kovachki2023neural}.
Among the most widely used approaches in this context are \emph{Neural Operators} (NOs), which generalize classical neural networks to the broader task of learning operators, i.e., mappings between (potentially) infinite-dimensional function spaces \citep{Kovachki2023NeuralOL,li2021fourier, qin2024betterunderstandingfourierneural, doi:10.1126/sciadv.abi8605, RAISSI2019686, li2020neuraloperatorgraphkernel, sharma2024graph}.


Despite their practical success in scientific computing, the theoretical understanding of Neural Operators (NOs) is still limited. Existing work has primarily focused on approximation properties \citep{huang2024operator, schwab2021deeplearninghighdimension, Marcati_2023, JMLR:v22:21-0806, kovachki2024operator}, while generalization results are comparatively scarce (see, e.g., \cite{kim2022boundingrademachercomplexityfourier, LARABENITEZ2024113168}). Recent progress has been made in the neural tangent kernel (NTK) regime, where \citet{nguyen2024optimalconvergenceratesneural} established minimax rates for operator learning. At the same time, gradient descent (GD) dynamics suggest a connection to random feature approximations of vector-valued reproducing kernels, but convergence guarantees for such methods are not yet available. To date, only Tikhonov regularization has been analyzed in this setting \citep{lanthaler2023error}. In this work, we contribute by deriving minimax rates for a broad class of spectral regularization schemes, including GD, thereby enabling the derivation of generalization bounds for NOs.

In addition to bridging this theoretical gap, random feature approximation (RFA) is of independent interest. Kernel methods remain state-of-the-art in many non-parametric statistical applications and provide an elegant framework for developing new theoretical insights \citep{spectral.rates, Lin_2020, zhang2024optimalitymisspecifiedspectralalgorithms}. However, their benefits come at a substantial computational cost, making them infeasible for large-scale datasets. Classical kernelized algorithms require storing the kernel Gram matrix $\mathbf{K} \in \mathbb{R}^{n \times n}$, with entries $\mathbf{K}_{i,j} = K(u_i, u_j)$ for kernel function $K(\cdot,\cdot)$ and data points $u_i, u_j$. This entails a memory cost of $O(n^2)$ and a time cost of up to $O(n^3)$, where $n$ denotes the dataset size \citep{kernellearning}.

RFA alleviates these costs by exploiting integral representations of kernels that can be approximated via finite sums of random features. For kernel ridge regression (KRR), this reduces memory and computational costs to $O(nM)$ and $O(nM^2)$, respectively, where $M$ is the number of random features \citep{features}. For gradient descent (GD), the cost becomes $O(nMt)$, with $t$ denoting the number of iterations, and can be further reduced by acceleration methods such as Heavy-Ball or Nesterov, which achieve the same generalization error as GD in only $\sqrt{t}$ iterations \citep{pagliana2019implicit}. Our analysis further suggests that $M$ should typically exceed $t$, which explains why random feature methods often perform best when combined with iterative regularization schemes. This observation highlights the need for rigorous theoretical guarantees in such settings.

For kernel ridge regression (KRR), the central question of how many random features are required to achieve optimal convergence rates has been studied extensively \citep{NIPS2007_013a006f, li2021unified, pmlrv119zhen20a, features, lanthaler2023error}. \citet{NIPS2007_013a006f} first established optimal rates for $M = O(n)$ random features in the case of real-valued kernels (rvk). This was later improved by \citet{features} to $M = O(\sqrt{n}\,\log n)$, and further extended to stochastic kernel ridge regression in \cite{SGDfeatures}. Most recently, \citet{lanthaler2023error} removed the logarithmic factor, showing that $M = O(\sqrt{n})$ suffices in the more general setting of vector-valued kernels (vvk). A detailed comparison is provided in Table~\ref{table3}. These results for $M$ hold in the well-specified case where the target function belongs to the RKHS. The question of how many random features are required for broader smoothness classes, where the target function may lie outside the RKHS, remains open.

\vspace{0.2cm}

\textbf{Contribution.} 
Our main motivation for studying vector-valued kernels and random feature approximation (RFA) is to derive generalization guarantees for Neural Operators (NOs). 

To the best of our knowledge, RFA with vector-valued kernels has previously been analyzed mainly in the setting 
of KRR. 
In this work, we develop a unified framework based on spectral filtering \citep{Caponetto}, which yields optimal convergence rates for a broad class of learning methods with either explicit or implicit regularization. 
This includes gradient descent and acceleration techniques, and it recovers as special cases the KRR results of \citet{features} for real-valued kernels and of \citet{lanthaler2023error} for vector-valued kernels. 

Our framework further accommodates kernels represented as sums of integral kernels, covering in particular operator-valued neural tangent kernels. 
This extension provides, for the first time, convergence rates for random feature methods with vector-valued kernels beyond KRR, thereby opening the door to rigorous statistical guarantees for NOs in the NTK regime.  

A key advantage of our approach is that both the convergence rates and the number of random features required for optimality are 
independent of the dimension of the input space. 
This makes the results directly applicable to NOs, where inputs are functions rather than finite-dimensional vectors. 
At the same time, the number of random features scales only quadratically with the feature dimension of the combined input representation per neuron, 
yielding generalization guarantees for NOs that combine minimax-optimal statistical rates with computational efficiency.

\begin{table}[h!]
\caption{Comparison of random feature requirements ($M$) for achieving generalization error of order $O(n^{-\frac{1}{2}})$. The last column indicates the smoothness classes where optimal rates are known; see Assumptions~\ref{ass:source}, \ref{ass:dim} for the meaning of $r,b>0$. [1]=\cite{NIPS2007_013a006f}, [2]=\cite{features}, [3]=\cite{lanthaler2023error}}
\label{table3}
\resizebox{\columnwidth}{!}{%
\begin{tabular}{|l|l|l|l|}
\hline References & $M$ & Method & Smoothness \\
\hline \hline 
[1] & $O(n)$ & KRR (rvk) & $r \in [0.5,1]$ \\
\hline 
[2] & $O(\sqrt{n}\log n)$ & KRR (rvk) & $r \in [0.5,1]$ \\
\hline 
[3] & $O(\sqrt n)$ & KRR (vvk) & $r=0.5$ \\
\hline 
Our & $O(\sqrt{n}\log n)$ & Spec. (vvk) & $2r+b>1$ \\
\hline 
\end{tabular}}
\end{table}

The rest of the paper is organized as follows. Section 2 introduces the setting, recalls key definitions in the context of random feature methods, and motivates the framework by linking it to Neural Operators. Section 3 presents and discusses our main results.   Section 4 concludes with a summary of our findings. Appendix A provides further details on learning with Neural Operators and includes numerical illustrations that support our theoretical results. All proofs are deferred to Appendix B.

\vspace{0.2cm}

{\bf Notation.} 
By $\cL(\cH_1, \cH_2)$ we denote the space of bounded linear operators between real separable Hilbert spaces $\cH_1$, $\cH_2$. 
We write $\cL(\cH, \cH) = \cL(\cH)$. For $\Gamma \in \cL(\cH)$, we denote by $\Gamma^*$ the adjoint operator. If $h \in \cH$, we write $h \otimes h := \langle \cdot, h \rangle h$. For a compact operator $\Gamma \in \cL(\mathcal{H})$, the trace is defined by $\operatorname{tr}(\Gamma) = \sum_{k=1}^\infty \langle \Gamma e_k, e_k \rangle,$ where $\{e_k\}_{k=1}^\infty$ is any orthonormal basis of $\mathcal{H}$. 
We denote by $\mathcal{F}(\mathcal{U},\mathcal{V})$ the space of measurable operators from $\mathcal{U}$ to $\mathcal{V}$.  
We write $L^2(\mathcal{U},\rho_{\mathcal{U}}):= L^2(\mathcal{U},\rho_{\mathcal{U}};\mathcal{V})$ for the $L^2$ space equipped with the norm $\|f\|^2_{L^2(\rho_\mathcal{U})}:=\int_{\mathcal{U}} \|f(u)\|_\mathcal{V}^2 \, d\rho_\mathcal{U}(u)$.  We let $[n] := \{1,\dots,n\}$ and denote the output vector by $\mathbf{v} = (v_1,\dots,v_n) \in \mathcal{V}^n$ with norm $\|\mathbf{v}\|_2^2 := \sum_{i=1}^n \|v_i\|_\mathcal{V}^2$.


\section{Mathematical Framework}
\label{sec:setting}

We consider an input space $(\mathcal{U}, ||\cdot||_{\mathcal{U}})$, where $\mathcal{U}$ is a Banach space, and an output space $(\mathcal{V}, ||\cdot||_{\mathcal{V}})$, where $\mathcal{V}$ is a separable Hilbert space. These assumptions facilitate the use of the theory of vector-valued kernels \citep{vvk1,vvk2}. The data space is given by $\mathcal{Z} = \mathcal{U} \times \mathcal{V}$, equipped with an unknown distribution $\rho$. We denote by $\rho_{\mathcal{U}}$ the marginal distribution on $\mathcal{U}$, and by $\rho(\cdot \mid u)$ the regular conditional distribution on $\mathcal{V}$ given $u \in \mathcal{U}$; see \citet{Shao_2003_book}.

Given a measurable operator $G: \mathcal{U} \to \mathcal{V}$ we further define the expected risk as 
\begin{equation}
\label{eq:expected-risk}
\cE(G) := \mbe[ \ell (G(u), v) ]\;,
\end{equation}  
where the expectation is taken w.r.t. the distribution $\rho$ and $\ell: \mathcal{V} \times \mathcal{V} \to \mbr_+$ is the least-square loss 
$\ell(v, v')=\frac{1}{2}\|v-v'\|_\mathcal{V}^2$. It is known that the global minimizer  of $\cE$ over the set of all measurable functions is 
given by the regression operator $G_\rho(u)= \int_{\mathcal{V}} v \rho(dv|u)$.  

We consider a standard statistical learning setting where we are given data $(u_j, v_j)_{j=1}^n$, 
sampled identically and independently with respect to $\rho$ on $\cU \times \cV$.


\subsection{Motivation: Generalization Bounds for Neural Operators}
\label{sec:motivation}

{\bf Shallow NOs.} To connect shallow Neural Operators (NOs) with vector-valued kernels, it is useful to recall their definition and highlight how their training dynamics give rise to neural tangent kernels (NTKs). 
Following \citet{nguyen2024optimalconvergenceratesneural}, the class of two-layer NOs is defined as follows.  
Let $\mathcal{U}$ denote the function input space, mapping from the measure space $(\mathcal{X}, \rho_x)$ to $\mathcal{Y} \subset \mathbb{R}^{d_y}$, and let $\mathcal{V}$ be the target function space, containing measurable functions mapping from $(\mathcal{X}, \rho_x)$ to 
$\tilde{\mathcal{Y}} \subset \mathbb{R}$. 
For a network of width $M \in \mathbb{N}$, let $\sigma:\mathbb{R}\to\mathbb{R}$ be an activation function acting pointwise, and let $A:\mathcal{U}\to \mathcal{F}(\mathcal{X},\mathbb{R}^{d_k})$ be a continuous operator.  We then define the class of 
shallow NOs by \label{NOCLASS}
\begin{align*}
\mathcal{F}_{M} &:= \Bigl\{\, G_\theta : \mathcal{U} \to \mathcal{V} \;\Big|\; 
G_\theta(u)(x) \\
&= \frac{\big\langle a,\; \sigma\!\big(B_1 A(u)(x) + B_2 u(x) + B_3 c(x)\big) \big\rangle}{\sqrt{M}}, \\
\theta &= (a, B_1, B_2, B_3) \\
&\in 
\mathbb{R}^M \times \mathbb{R}^{M \times d_k} \times 
\mathbb{R}^{M \times d_y} \times \mathbb{R}^{M \times d_b}, \\
&\; c : \mathcal{X} \to \mathbb{R}^{d_b} 
\Bigr\}.
\end{align*}
where $\langle \cdot, \cdot \rangle$ denotes the Euclidean inner product. We collect all parameters in $\theta = (a, B_1,B_2,B_3 )=(a, B )$, with $B=(B_1,B_2,B_3)\in\mathbb{R}^{M\times\tilde{d}}, $ and \emph{feature dimension} $\tilde{d}:=d_k+d_y+d_b$. The goal is to minimize the expected risk 
\eqref{eq:expected-risk} over the set $\mathcal{F}_{M}$, i.e., 
\[  \min_{G_\theta \in \cF_M} \mbe[ \ell (G_\theta(u), v) ]\;.\]

{\bf Relation to vector-valued kernels via the NTK.} The connection between kernel methods and standard neural networks is established via the neural tangent kernel (NTK), see \citet{jacot2018neural, lee2019wide}. For fully connected neural networks, it is known that under gradient descent (GD) training and in the infinite-width limit, the network function linearizes around initialization, and the dynamics are governed by the NTK.

In complete analogy, the training dynamics of shallow NOs correspond to kernel gradient descent in the reproducing kernel Hilbert space (RKHS) induced by a vector-valued kernel. We define for neural operators the NTK feature map 
\[
\Phi^M : \mathcal{U} \to \mathcal{F}(L^2(\mathcal{X},\rho_x),\Theta), 
\qquad \Phi^M(u) := \Phi^M_u,
\]
with random initialization $\theta_0 \in \Theta$ and 
\[
\Phi^M_u(v) := \nabla_\theta \big\langle G_{\theta_0}(u), v \big\rangle_{L^2(\rho_x)} \,, 
\qquad v \in L^2(\mathcal{X},\rho_x).
\]
For any $u,\tilde u \in \mathcal{U}$, the vector-valued kernel $K_M:  L^2(\mathcal{X},\rho_x) \to  L^2(\mathcal{X},\rho_x)$ is defined by 
\[ K_M(u,\tilde u) := (\Phi^M_u)^* \Phi^M_{\tilde u} \] 
and expands as 
\begin{align*}
K_M(u,\tilde u) 
&= \frac{1}{M} \sum_{m=1}^{M} \psi_m (u)  \otimes \psi_m (\tilde u) \\
&\quad + \frac{1}{M} \sum_{m=1}^{M}\sum_{j=1}^{\tilde{d}} \psi'_{m,j} (u) \otimes\psi'_{m,j} (\tilde u).
\end{align*}
Here, for \(u\in\mathcal{U}\) and \(x\in\mathcal{X}\), we set 
\begin{align*}
J(u)(x) &:= \big(A(u)(x),\,u(x),\,c(x)\big)^\top ,  \\
\psi_m (u) &:=   \sigma\!\big(\langle b_m^{(0)}, J(u)\rangle \big)  , \\
\psi'_{m,j} (u) &:=  \sigma'\!\big(\langle b_m^{(0)}, J(u)\rangle \big) J(u)^{(j)}  
\end{align*} 
and denote the collection of initial weight vectors by \(B^{(0)}=(b_1^{(0)},\dots,b_M^{(0)})^\top \in \mathbb{R}^{M\times \tilde d}\), see \citet{nguyen2024optimalconvergenceratesneural}.

\vspace{0.2cm}
{\bf Relation to Random Feature Approximation.} 
The representation of $K_M$ shows that the NTK of shallow NOs admits the form of a random feature approximation. 
Indeed, each term $\psi_m(u)\otimes\psi_m(\tilde u)$ and $\psi'_{m,j}(u)\otimes\psi'_{m,j}(\tilde u)$ 
constitutes a nonlinear random feature, sampled through the random initialization $b_m^{(0)}$. 
Thus, $K_M$ is a Monte–Carlo approximation of the limiting kernel 
\begin{align}
\label{eq:NTK-ex}
& K(u,\tilde u) := \\ \nonumber 
& \quad \mathbb{E}_{\theta_0} \Big[\psi(u)\otimes \psi(\tilde u) 
+ \sum_{j=1}^{\tilde d} \psi'_{j}(u)\otimes \psi'_{j}(\tilde u)\Big],
\end{align}
where $\psi(u)= \sigma\!\big(\langle b^{(0)}, J(u)\rangle \big)$, 
$\psi'_{j}(u)= \sigma'\!\big(\langle b^{(0)}, J(u)\rangle \big) J(u)^{(j)}$ and  
the expectation is taken w.r.t. the initialization distribution of $\theta_0$, see 
\citet{nguyen2024optimalconvergenceratesneural}, Proposition 2.3. 

Consequently, training shallow NOs with gradient descent in the NTK regime 
is equivalent to performing kernel gradient descent in the vector-valued RKHS 
associated with $K$, using $K_M$ as a random feature approximation. 
This interpretation provides the bridge between Neural Operators and 
the statistical theory of random features for vector-valued kernels.

\vspace{0.2cm}
{\bf Generalization Bound.} Assume a NO is trained with gradient descent by empirical risk minimization over the 
class $\cF_M$ and let $\theta_t$ denote the parameter update after $t \in \mbn$ iterations.  
The excess risk of the corresponding neural operator $G_{\theta_t}\in\mathcal{F}_M$ can be decomposed as 
\begin{align}\label{NOdecomp}
&\|G_{\theta_t} - G_\rho\|^2_{L^2(\rho_{\cU})}\\
&\lesssim 
 \|G_{\theta_t}  - F_t^M\|^2_{L^2(\rho_{\cU})} +  
  \| F_t^M-  G_\rho\|^2_{L^2(\rho_{\cU})}  \;,\nonumber
\end{align}
where $F_t^M$ denotes the random-feature estimator associated with the vector-valued kernel $K_M$, trained with the same GD dynamics. 
The first term captures the discrepancy between the finite-width NO and its NTK-based random feature approximation, while the second term measures the generalization error of the random feature method itself. {\bf The central objective of this paper is to establish optimal learning guarantees for the second term, and to extend the analysis to a broader class of regularization schemes and kernels. }


\subsection{Kernel Methods, Random Features, and Regularization: General Setup}

{\bf Vector-valued Kernels.} Let $K : \mathcal{U} \times \mathcal{U} \longrightarrow \mathcal{L}(\mathcal{V})$ denote 
a reproducing $\cV$-\emph{valued kernel of positive type}, where $\mathcal{L}(\mathcal{V})$ denotes the Banach space of bounded linear operators 
on $\mathcal{V}$.  
For $u \in \mathcal{U}$, we write $K_u : \mathcal{V} \to \mathcal{F}(\mathcal{U}, \mathcal{V})$ for the operator that maps $v \in \mathcal{V}$ to the function $K_uv \in \mathcal{F}(\mathcal{U},\mathcal{V})$ given by
\begin{align}\label{k_u}
(K_uv)(\tilde u) = K(\tilde u,u)v, \qquad \tilde u \in \mathcal{U}.
\end{align}
By $\cH$ we denote the associated unique $\cV$-valued RKHS on $\cU$, which can be continuously included into $\mathcal{F}(\mathcal{U}, \mathcal{V})$. We assume that $K$ is \emph{2-bounded}, implying that the inclusion $S:\cH \hookrightarrow L^2(\cU, \rho_{\cU}; \cV)$ is a bounded linear map.  
A brief review of the key definitions related to vector-valued kernels is provided in Appendix~A, see also \citet{vvk1, vvk2}.

\begin{assumption}[Kernel]
\label{ass:kernel}
Assume that the kernel $K$ admits an integral representation of the form 
\begin{equation}
\label{eq:kernel-rep}
K(u,\tilde{u}) = \sum_{i=1}^p \int_\Omega \varphi_i(u,\omega) \otimes \varphi_i(\tilde{u},\omega)\, d\pi(\omega),
\end{equation}
where $\varphi_i : \mathcal{U} \times \Omega \to \mathcal{V}$ for $i=1,\dots,p$, and $(\Omega,\pi)$ is a probability space.  
Moreover, assume that for all $u \in \mathcal{U}$, 
\begin{equation}
\label{eq:RFbound}
\sum_{i=1}^p \|\varphi_i(u,\omega)\|_{\mathcal{V}}^2 \leq \kappa^2 
\quad \pi-\text{almost surely }. 
\end{equation}
\end{assumption}
Note that Assumption~\ref{ass:kernel} implies the Hilbert–Schmidt bound 
$\|K(u,\tilde{u})\|_{HS} \leq \kappa^2$ $\rho_{\cU}$-almost surely. Examples of \eqref{eq:kernel-rep} include the \emph{Gaussian kernel} and \emph{Random Fourier features} \citep{NIPS2007_013a006f, features}. 
In contrast to these works, we allow for the additional sum over $i$, which enables us to cover special cases of Neural Tangent Kernels \eqref{eq:NTK-ex}, see also \citet{jacot2018neural, nitanda2020optimal, Li21, Munteanu22, Oymak, nguyen2023neurons}.

\vspace{0.2cm}
{\bf Random Feature Approximations.}
The idea of RFA is to approximate kernels that admit an integral representation \eqref{eq:kernel-rep} by a kernel represented by a finite sum, i.e., $K(u,\tilde{u}) \approx K_M(u,\tilde{u})$ for $u,\tilde{u} \in \mathcal{U}$ and $M \in \mbn$, where
\begin{align*}
K_M(u,\tilde{u}) := \sum_{i=1}^p \frac{1}{M} \sum_{m=1}^M 
\varphi_i(u,\omega_m) \otimes \varphi_i(\tilde{u},\omega_m).
\end{align*}
Here, $\varphi_i : \mathcal{U} \times \Omega \to \mathcal{V}$ for some probability space $(\Omega,\pi)$, 
and $\{\omega_m\}_{m=1}^M$ are drawn i.i.d.\ from $\pi$. 

The associated RKHS is denoted by $\cH_M$. Under Assumption \eqref{eq:RFbound}, the inclusion $\cS_M: \cH_M \hookrightarrow L^2(\cU, \rho_{\cU}; \cV)$ is a bounded linear map. 


\vspace{0.2cm}

The main benefit of RFA is that it renders kernel methods computationally feasible on large datasets while preserving their statistical guarantees. Classical kernel methods require storing the full $n\times n$ Gram matrix and solving linear systems at cost $O(n^2)$ in memory and $O(n^3)$ in time. In contrast, RFA replaces the kernel with an explicit $M$-dimensional feature map, reducing the computational cost to $O(nM^2)$ (for ridge regression) or $O(nMt)$ (for $t$ gradient descent iterations), with only $O(nM)$ memory. Statistically, it has been shown that for many kernels $M=O(\sqrt{n})$ features suffice to achieve minimax-optimal learning rates. Thus, RFA provides a scalable and flexible framework for kernel methods.

\vspace{0.2cm}

{\bf Regularization.} 
Since the expected risk \eqref{eq:expected-risk} cannot be minimized directly, the standard procedure is 
empirical risk minimization (ERM) over the hypothesis space $\mathcal{H}$,
\[
\min_{F \in \mathcal{H}} \; \widehat{\mathcal{E}}(F), 
\qquad  
\widehat{\mathcal{E}}(F) = \frac{1}{n}\sum_{j=1}^n \ell\!\big(F(u_j), v_j\big).
\]
However, direct ERM in an RKHS setting is typically ill-posed. To avoid overfitting and to obtain consistent estimators, 
\emph{regularization} is introduced. 

\vspace{0.1cm}

\begin{definition}[Regularization function] 
Let $\phi :(0,1]\times [0,1]\to\mathbb{R}$ and set $\phi_\lambda(t)=\phi(\lambda,t)$. 
The family $\{\phi_\lambda\}_{\lambda}$ is called a family of \emph{regularization functions} if there exist constants $D,E,c_0>0$ such that for all $0<\lambda \leq 1$:  
\begin{align}
\text{i)} \;\; & \sup_{0<t\leq 1} |t\phi_\lambda(t)| \leq D, \label{def.phi}\\
\text{ii)} \;\; & \sup_{0<t\leq 1} |\phi_\lambda(t)| \leq \tfrac{E}{\lambda},\\
\text{iii)} \;\; & \sup_{0<t\leq 1} |r_\lambda(t)| \leq c_0,\;\;  
 r_\lambda(t):=1-t\phi_\lambda(t). \label{residual}
\end{align}
\end{definition}

This family of methods, known as \emph{spectral regularization}, encompasses both explicit regularization, such as Tikhonov regularization, and implicit regularization through iterative schemes, including gradient descent and accelerated methods.  
Originally developed for (statistical) inverse problems \citep{engl1996regularization}, these techniques have since been applied in machine learning, in particular to non-parametric least-squares regression 
\citep{Caponetto, bauer2007regularization, Muecke2017op.rates, lin2020optimal}.  

It has been shown in \cite{10.1162/neco.2008.05-07-517, Muecke2017op.rates} that attainable learning rates are essentially determined by the \emph{qualification} of the regularization $\{\phi_\lambda\}_{\lambda}$, i.e., the largest $\nu>0$ such that for all $q\in[0,\nu]$ and $0<\lambda \leq 1$:  
\begin{equation}\label{c_r}
\sup_{0 < t \leq 1}\, |r_{\lambda}(t)|\, t^{q} \leq c_{q}\,\lambda^{q},
\end{equation}
for some constant $c_q > 0$. 

\vspace{0.2cm}

A principled approach is to exploit the spectral structure of the empirical operators 
$\widehat{\Sigma}_M: \cH_M \to \cH_M$ and $\widehat{\cS}^*_M: \cV^n \to \mathcal{H}_M$, defined as
\begin{align*}
\widehat{\Sigma}_M &= \frac{1}{n}\sum_{j=1}^n K_{M,u_j}K_{M,u_j}^*, \\
\widehat{\cS}^*_M \mathbf{v} &= \frac{1}{n}\sum_{j=1}^n K_{M,u_j} v_j.
\end{align*}
With these operators, spectral regularization estimators combined with RFA take the form
\begin{equation}
\label{def:estimator}
F^M_\lambda = \phi_\lambda(\widehat{\Sigma}_M)\, \widehat{\cS}^*_M \mathbf{v} \; \in \; \cH_M.
\end{equation}

\section{Main Results}
\label{sec:main-results}

\subsection{Assumptions and Main Results}

In this section we formulate our assumptions and state our main results. 

\begin{assumption}[Data Distribution]
\label{ass:input}
 There exists positive constants $Q$ and $Z$ such that for all $l \geq 2$ with $l \in \mathbb{N}$,
$$
\int_{\mathcal{V}}\|v\|_\mathcal{V}^l\;  d \rho(v \mid u) \leq \frac{1}{2} l ! Z^{l-2} Q^2
$$
$\rho_\mathcal{U}$-almost surely.
\end{assumption}
This assumption is satisfied, for example, if $v$ is bounded almost surely. 
It further implies that the regression operator $G_\rho$ is bounded almost surely, since 
\begin{align*}
\left\|G_\rho(u)\right\|_\mathcal{V} &\leq \int_{\mathcal{V}}\|v\|_\mathcal{V}\; d \rho(v \mid u) \\
&\leq\left(\int_{\mathcal{V}}\|v\|_{\mathcal{V}}^2\; d \rho(v \mid u)\right)^{\frac{1}{2}} \leq Q\,.
\end{align*}

To characterize the smoothness of $G_\rho$ relative to the kernel, 
we impose a so-called \emph{source condition}. 
This condition links $G_\rho$ to the spectral properties of the kernel integral operator 
and plays a central role in determining the attainable learning rates.

Denote by $\mathcal{L}:  L^2(\mathcal{U} , \rho_\mathcal{U})\to  L^2(\mathcal{U} , \rho_\mathcal{U})$ the kernel integral operator associated to $K$, i.e. 
\[ \mathcal{L}G = \int_{\mathcal{U}}  K_{u} G(u) \;\rho_\mathcal{U}(d u) . \]

\begin{assumption}[Source Condition]
\label{ass:source}
Let $R>0$, $r>0$. We assume 
$G_\rho  = \mathcal{L}^r H, \label{hsource}\,$ 
for some $H \in L^2(\mathcal{U} , \rho_\mathcal{U})$, 
satisfying $\|H\|_{L^2(\rho_\mathcal{U})} \leq R$ . 
\end{assumption} 
This condition links $G_\rho$ to the spectral properties of the kernel integral operator 
and plays a central role in determining the attainable learning rates. 
The parameter $r>0$ quantifies the degree of smoothness of $G_\rho$: 
larger values of $r$ correspond to higher regularity. In particular, the case $r=\tfrac{1}{2}$ corresponds to the well-specified setting where 
$G_\rho$ lies in $\cH$, 
while $r>\tfrac{1}{2}$ reflects additional smoothness, 
and $r<\tfrac{1}{2}$ corresponds to a misspecified setting. For more details, we refer to e.g. \citet{bauer2007regularization, lin2020optimal}. 


While the source condition controls the regularity of the target function, 
a complementary assumption is required to capture the complexity of the hypothesis space. 
This is typically expressed in terms of the \emph{effective dimension}, 
which measures the capacity of the RKHS relative to the kernel operator spectrum. 

\begin{assumption}[Effective Dimension]
\label{ass:dim} 
For some $b \in [0,1]$ and $c_b > 0$, assume that for all $\lambda > 0$ the operator $\mathcal{L}$ satisfies
\begin{align}
\label{effecDim}
\mathcal{N}(\lambda) := \operatorname{tr} \; \left(\mathcal{L}(\mathcal{L}+\lambda I)^{-1}\right) \leq c_b \lambda^{-b}.
\end{align}
Moreover, we assume that $2r+b > 1$.
\end{assumption}
Here, $\mathcal{N}(\lambda)$ is the \emph{effective dimension} of the kernel, which quantifies the number of effective degrees of freedom of the hypothesis space. Intuitively, it reflects the decay of the eigenvalues of $\mathcal{L}$: smaller values of $b$ correspond to faster decay (lower capacity), while larger values of $b$ indicate slower decay and hence higher complexity. The condition \eqref{effecDim} is always satisfied with $b=1$, since $\mathcal{L}$ is trace class and its eigenvalues $\{\mu_i\}$ satisfy $\mu_i \lesssim i^{-1}$. If, more generally, the eigenvalues decay polynomially as $\mu_i \sim i^{-c}$ with $c>1$, then \eqref{effecDim} holds with $b=1/c$; if $\mathcal{L}$ has finite rank, then $b=0$. The case $b=1$ is often called the \emph{capacity-independent} case. Smaller values of $b$ allow for faster convergence rates of the learning algorithms.  


\vspace{0.2cm}

The choice of the regularization parameter $\lambda$ and the number of random features $M$ in \eqref{def:estimator} 
is crucial for balancing approximation, estimation, and optimization errors (see Appendix~B). 
In practice, both parameters are determined as functions of the sample size $n$, 
in order to guarantee optimal statistical performance. 
The following result establishes that, under the source and capacity assumptions, 
our RF estimator achieves the minimax-optimal convergence rate. 
Moreover, it provides explicit conditions on $\lambda_n$ and $M_n$ 
that ensure the desired statistical guarantees. The proof is provided in Appendix~B.

\begin{theorem}
\label{theo1}
Suppose Assumptions~\ref{ass:input}--\ref{ass:dim} hold. Let $\{\phi_\lambda\}_\lambda$ be a family of regularization 
functions with qualification $\nu >0$. Let $\delta \in (0,1)$ and choose
\[
\lambda_n = C\, n^{-\frac{1}{2r+b}} \log^3\!\left(\tfrac{2}{\delta}\right).
\]
Then, with probability of at least $1-\delta$, the RF estimator \eqref{def:estimator} satisfies
\begin{align*}
\|G_\rho - \mathcal{S}_{M_n} F_{\lambda_n}^{M_n}\|_{L^2(\rho_\mathcal{U})}
\;\leq\; \bar{C}\, n^{-\frac{r}{2r+b}} \log^{3r+1}\!\left(\tfrac{1}{\delta}\right),
\end{align*}
provided that $\nu \geq r \vee 1$, $n \geq n_0 := \exp\!\left(\tfrac{2r+b}{2r+b-1}\right)$, 
and the number of random features satisfies
\begin{align*}
M_n \;\geq\; p \cdot \tilde{C}\cdot  \log(n) \cdot 
\begin{cases}
n^{\frac{1}{2r+b}}, & r \in (0,\tfrac{1}{2}), \\[0.5ex]
n^{\frac{1+b(2r-1)}{2r+b}}, & r \in [\tfrac{1}{2},1], \\[0.5ex]
n^{\frac{2r}{2r+b}}, & r \in (1,\infty).
\end{cases}
\end{align*}
Here $C, \tilde{C}, \bar{C}$ are constants independent of $n, M, \lambda$.  
\end{theorem}

Since our framework encompasses operator-valued RFA trained via GD, 
Theorem~\ref{theo1} can be directly applied to derive generalization bounds for NOs. 
Recall the excess-risk decomposition for shallow NOs in \eqref{NOdecomp}.  
\citet{nguyen2024optimalconvergenceratesneural} showed that the first term is bounded by $O(\log n / M_n)$. 
Hence, when the number of neurons scales with the number of random features, i.e., 
$M_n = O\!\left(n^{\tfrac{2r}{2r+b}} \log n\right)$, 
Theorem~\ref{theo1} implies that NOs achieve the same minimax rates as non-parametric kernel methods.


\begin{corollary}[\citet{nguyen2024optimalconvergenceratesneural}, Theorem.~3.5]
\label{main:corr}
Suppose the assumptions of Theorem~\ref{theo1} hold. Let $G_{\theta_{T_n}}$ denote the NO as defined 
in Section \ref{sec:motivation}, where $T_n$ is the number of GD iterations. Let $M_n$ denote the network width. Assume   
\begin{align*}
\lambda_n \;&=\; T_n^{-1} \;=\; C\, n^{-\frac{1}{2r+b}}, \\
M_n \;&\geq\; \tilde{d}^2\cdot \tilde{C}\,B_{T_n}^6 \bigl(T^{2r}_n \vee T_n\bigr)\log^2 n \,,
\end{align*}
where $B_{T_n}>0$ bounds the parameter drift,
\[
\| \theta_t - \theta_0 \|_2 \;\leq\; B_{T_n} 
\quad\text{for all } t\in[T_n],
\]
and $C,\tilde{C}$ are positive constants independent of $n,M_n,T_n,B_{T_n}$.  
Then, with probability of at least $1-\delta$,
\[
\| G_{\theta_{T_n}} - G_\rho \|_{L^2(\rho_{\mathcal{U}})}
\;\leq\;
\bar{C}\,
n^{-\frac{r}{2r+b}}
\log^3\!\bigl(2/\delta\bigr),
\]
for some constant $\bar{C}>0$ independent of $n,M_{n},T_{n},B_{T_n}$.
\end{corollary}

Further details on the training, the initialization of $\theta_{0}$, and the above corollary are provided in Appendix~A.


\subsection{Discussion}

\paragraph{Discussion of Theorem \ref{theo1}.} 
Theorem~\ref{theo1} establishes minimax rates for a broad class of spectral filtering methods. 
In the well-specified case ($r=\tfrac{1}{2}$, $b=1$), achieving a squared $L^2$-error bound of order $O(1/\sqrt{n})$ requires 
$t_n = 1/\lambda_n = O(\sqrt{n})$ iterations and $M_n = O(\sqrt{n}\log n)$ random features. 
For smoother target functions with regularity $r \geq 1$, the optimal convergence rate is attained with 
$t_n = O(n^{\frac{1}{2r+1}})$ iterations, but requires 
$M_n = O(t_n^{2r}\log n)$ random features. 
This highlights an interesting trade-off: higher smoothness reduces the number of necessary iterations but 
increases the number of random features required for optimal generalization.  

In contrast, in the misspecified case $r < \tfrac{1}{2}$, the attainable rate of order 
$O(n^{-\frac{r}{2r+1}})$ is slower, reflecting the limited regularity of the target function. 
In this regime, the required number of random features is only 
$M_n = O( n^{\frac{1}{2r+1}}\log n)$, which is significantly smaller than in the well-specified or smooth cases.  
Overall, Theorem~\ref{theo1} shows that random feature methods with spectral regularization achieve the same minimax-optimal rates as exact kernel methods \citep{Caponetto, Muecke2017op.rates}, while offering improved computational scalability.

\paragraph{Comparison with prior work.} 
Compared to the results of \citet{features, lanthaler2023error}, our work extends the analysis from KRR to general spectral filtering methods and establishes optimal convergence rates for all smoothness levels $r<\tfrac{1}{2}$ satisfying $2r+b>1$ (the \emph{easy learning regime}). 
With respect to the number of required random features, we recover the same order as \citet{features}, namely $M_n=O(\sqrt{n}\log n)$. 
The analysis in \citet{lanthaler2023error} is based on slightly different source assumptions, which coincide with ours in the well-specified case. 
Using a \emph{random kitchen sinks} approach \citep{NIPS2008_0efe3284}, they further showed that the logarithmic factor can be removed, proving that $M_n=O(\sqrt{n})$ random features suffice to achieve optimal rates. 
However, their results do not exploit prior knowledge about the effective dimension and therefore only establish optimal rates in the well-specified setting $b=1$, $r=\tfrac{1}{2}$.

\paragraph{NNs and NOs.} 
A connection between learning with neural networks and random feature approximation (RFA) was already observed in 
\citet{yehudai2019power}. Roughly speaking, they noted that learning with neural networks is possible whenever 
learning with random features is possible. At the same time, they showed that neural networks cannot be used to learn even a single 
ReLU neuron under Gaussian inputs in $\mathbb{R}^d$ with $\mathrm{poly}(d)$ weights, unless the network size (or the magnitude of its weights) 
is exponentially large in $d$. For smoother activations in the NTK regime, \citet{nguyen2023neurons} improved on this result 
by showing that optimality can be achieved with only a polynomial number of random features in both the input dimension $d$ 
and the sample complexity $n$. 

Our approach extends these insights to the operator-valued setting relevant for Neural Operators (NOs). 
In contrast to the vector-input case, our rates are \emph{dimension-free in the input space} $\cU$. 
However, the sum structure of the kernel representation \eqref{eq:kernel-rep} introduces a linear dependence on 
the number of summands $p$, as reflected in Theorem~\ref{theo1}. For NOs, the input space $\cU$ is typically a function space, 
e.g., continuous mappings from $\cX$ to $\mathbb{R}^{d_y}$. In this case, the output dimension $d_y$ of the input functions 
enters the feature dimension $\tilde d = d_k + d_y + d_b$, which directly determines the number of required random features. 
This results in an overall dependence of order $\tilde d^2$ (see Section~\ref{sec:motivation}) in our bounds, as stated in 
Corollary~\ref{main:corr}. Thus, our results reveal a clear trade-off: generalization rates are independent of the (possibly infinite) 
dimension of $\cU$, but the computational cost scales quadratically with the feature dimension $\tilde d$.

\section{Conclusion}

We developed a unified spectral filtering framework for RFA with vector-valued kernels, 
motivated by the goal of deriving generalization guarantees for NOs. 
Our analysis extends beyond KRR to a broad class of learning algorithms with explicit or implicit regularization, and recovers previous results as special cases. 
A key advantage of our approach is that both convergence rates and feature requirements are dimension-free in the (possibly infinite) input space, 
making the results directly applicable to NOs. At the same time, our bounds scale only quadratically with the feature dimension per neuron, 
providing the first minimax-optimal guarantees for NOs that combine statistical efficiency with computational tractability.  

The main theoretical result, Theorem~\ref{theo1}, establishes minimax rates for RFA under standard source and capacity assumptions, 
matching those of exact kernel methods while requiring significantly fewer resources. 
Our discussion highlights trade-offs between smoothness, iteration complexity, and the number of random features, 
as well as the contrast between well-specified, smooth, and misspecified regimes. 
Compared to prior work, our framework delivers optimal rates for all $r<\tfrac{1}{2}$ in the easy learning regime and 
covers operator-valued neural tangent kernels, linking neural networks, RFs, and NOs within a single theoretical setting. 
An interesting direction for future work is to investigate whether the quadratic dependence on the feature dimension $\tilde d$ 
can be further reduced, and to extend the analysis beyond the NTK regime to deeper architectures.

\bibliography{bib_iteration}

\clearpage
\appendix

\thispagestyle{empty}
\onecolumn

\input{AppA}

\input{AppB}

\input{T.U}

\end{document}

%% file: AppA.tex
\section{Additional Material}\label{AppA}

This appendix begins by reviewing the key definitions related to vector-valued kernels. We then provide additional background on Neural Operators and present a more detailed version of Corollary~\ref{main:corr}, accompanied by a brief discussion of its implications. Finally, we include numerical illustrations that support our main theoretical result, Theorem~\ref{theo1}.

\subsection{Preliminaries on Vector-Valued Kernels}

The classical theory of real-valued reproducing kernel Hilbert spaces \citep{Ingo}, including fundamental results such as Mercer's theorem, has been extended to the vector-valued setting in \citet{vvk1,vvk2}. These extensions provide a rigorous mathematical foundation for analyzing operator learning problems within a kernel framework \citep{minh2016operatorvaluedbochnertheoremfourier,brault2016randomfourierfeaturesoperatorvalued,mollenhauer2024learninglinearoperatorsinfinitedimensional,lanthaler2023error}, and have become standard tools in the operator learning literature. For completeness, we briefly recall the key definitions below.

\medskip
\noindent
Let $\mathcal{U}$ be a topological space and let $\mathcal{V}$ be a separable Hilbert space. A map 
\[
K : \mathcal{U} \times \mathcal{U} \to \mathcal{L}(\mathcal{V}),
\]
where $\mathcal{L}(\mathcal{V})$ denotes the space of bounded linear operators on $\mathcal{V}$, is called a \emph{$\mathcal{V}$-reproducing kernel} on $\mathcal{U}$ if, for any finite set of points $u_1, \ldots, u_N \in \mathcal{U}$ and vectors $v_1, \ldots, v_N \in \mathcal{V}$, it holds that
\[
\sum_{i,j=1}^N \left\langle K(u_i, u_j) v_j, v_i \right\rangle_{\mathcal{V}} \ge 0.
\]
This condition is the natural generalization of positive definiteness from the scalar- to the vector-valued setting.

\medskip
\noindent
For each $u \in \mathcal{U}$, we define the linear operator 
\[
K_u : \mathcal{V} \to \mathcal{F}(\mathcal{U}; \mathcal{V}),
\]
where $\mathcal{F}(\mathcal{U}; \mathcal{V})$ denotes the space of measurable $\mathcal{V}$-valued functions on $\mathcal{U}$. Its action on $v \in \mathcal{V}$ is given by
\[
(K_u v)(\tilde{u}) = K(\tilde{u}, u)v \qquad \text{for all } \tilde{u} \in \mathcal{U}.
\]

\medskip
\noindent
Given a $\mathcal{V}$-reproducing kernel $K$, there exists a unique Hilbert space $\mathcal{H}_K \subset \mathcal{F}(\mathcal{U}; \mathcal{V})$ such that 
\[
K_u \in \mathcal{L}(\mathcal{V}, \mathcal{H}_K) \quad \text{for all } u \in \mathcal{U},
\]
and for every $F \in \mathcal{H}_K$,
\[
F(u) = K_u^* F \qquad \text{for all } u \in \mathcal{U},
\]
where $K_u^*: \mathcal{H}_K \to \mathcal{V}$ denotes the adjoint of $K_u$. This is the vector-valued analogue of the classical reproducing property. In particular, it implies
\[
K(u,\tilde{u}) = K_u^* K_{\tilde{u}}.
\]

\medskip
\noindent
The space $\mathcal{H}_K$ is called the \emph{vector-valued reproducing kernel Hilbert space (RKHS)} associated with $K$, and it is given by
\[
\mathcal{H}_K = \overline{\mathrm{span}}\left\{ K_u v \,\big|\, u \in \mathcal{U},\, v \in \mathcal{V} \right\}.
\]

\medskip
\noindent
Finally, a reproducing kernel $K : \mathcal{U} \times \mathcal{U} \to \mathcal{L}(\mathcal{V})$ is called a \emph{Mercer kernel} if $\mathcal{H}_K$ is a subspace of $\mathcal{C}(\mathcal{U}; \mathcal{V})$, the space of continuous $\mathcal{V}$-valued functions on $\mathcal{U}$.

\input{AppNOs}

\subsection{Numerical Illustration}
\label{sec:numerics}

We analyze the behavior of kernel gradient descent with respect to the real-valued NTK. In our simulations, we use $n = 5000$ training and test samples drawn from two datasets: (i) a standard normal distribution with input dimension $d = 1$, and (ii) a subset of the SUSY\footnote{\url{https://archive.ics.uci.edu/ml/datasets/SUSY}} classification dataset with input dimension $d = 14$. All reported results are averaged over 50 independent runs of the algorithm. 

Our theoretical analysis suggests that a number of random features of order $M = O(\sqrt{n}\, p)$, where $p = d + 2$, is sufficient to achieve optimal learning performance. Indeed, Figure~\ref{F1} shows that for both datasets, once $M$ exceeds a threshold of order $O(\sqrt{n}\, p)$ and the number of GD iterations $T$ is fixed, further increasing $M$ does not lead to any improvement in the test error.

\begin{figure}[ht]
\centering
\includegraphics[width=0.42\columnwidth, height=0.19\textheight]{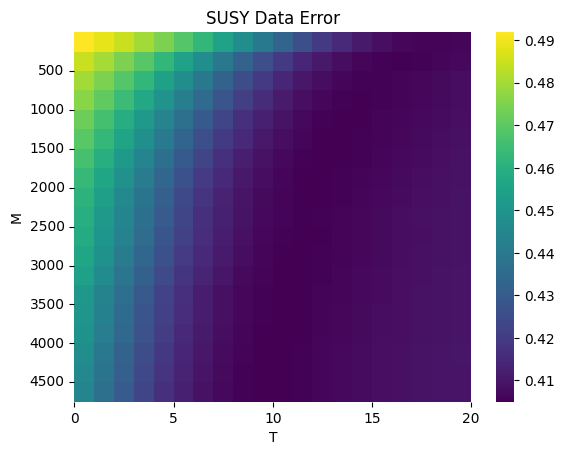}
\includegraphics[width=0.42\columnwidth, height=0.19\textheight]{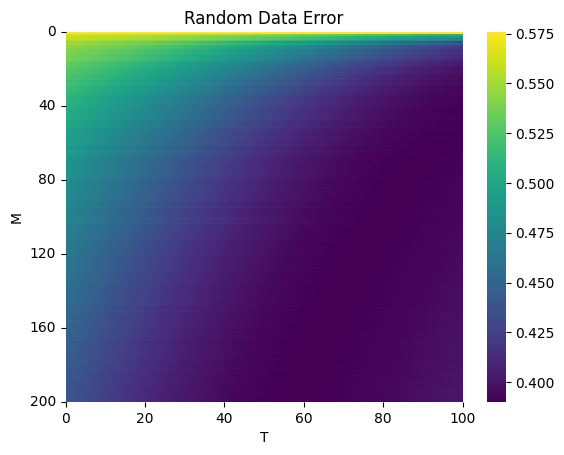}
\caption{Heat plot of the test-error for different numbers of RF $M$ and iterations $T$.\\ 
\label{F1}
{\bf Left:} Error of SUSY data set.
{\bf Right:} Error of random data set.}
\end{figure}

%% file: AppNOs.tex
\subsection{Learning with Neural Operators}

In this section we present a detailed version of Corollary~\ref{main:corr}, based on \citet[Theorem~3.5]{nguyen2024optimalconvergenceratesneural}. For completeness, we briefly recall the neural operator learning setup considered in their work, including the training procedure via gradient descent.

\paragraph{Neural Operator Setting.}
We consider the general operator learning framework introduced in Section \ref{sec:setting}, where the input and output spaces $\mathcal{U}$ and $\mathcal{V}$ are viewed as infinite-dimensional function spaces. To enable practical training based on function-valued data, an additional discretization step is required. This involves so-called \emph{second-stage samples} drawn from the input domain of the functions.

Specifically, let $\mathcal{X} \subseteq \mathbb{R}^{d_x}$ denote the input domain, and let $\mu$ be an unknown probability measure on $\mathcal{X}$. The spaces $\mathcal{U}$ and $\mathcal{V}$ consist of functions mapping $\mathcal{X} \to \mathcal{Y} \subset \mathbb{R}^{d_y}$ and $\mathcal{X} \to \tilde{\mathcal{Y}} \subset \mathbb{R}$, respectively. We observe $\nU$ i.i.d.\ first-stage samples
\[
\bigl\{(u_i,v_i)\bigr\}_{i=1}^{\nU} \subset \mathcal{U} \times \mathcal{V},
\]
each evaluated at $\nX$ i.i.d.\ second-stage samples $(x_1,\dots,x_{\nX}) \in \mathcal{X}^{\nX}$. These discretized evaluations are then used to train a shallow neural operator via gradient descent.

\paragraph{Gradient Descent and Initialization.}
Recall the class of shallow NOs introduced in Section~\ref{NOCLASS}. Following \citet{nguyen2024optimalconvergenceratesneural}, the parameters are trained by gradient descent on the empirical loss computed from the first-stage samples evaluated at the second-stage points:
\begin{align}
\theta_{t+1}^j 
&= \theta_t^j - \alpha \,\partial_{\theta^j} \widehat{\mathcal{E}}\bigl(G_{\theta_t}\bigr) \nonumber\\
&= \theta_t^j - \frac{\alpha}{\nU} \sum_{i=1}^{\nU} 
\Bigl\langle G_{\theta_t}(u_i) - v_i, \, \partial_{\theta^j} G_{\theta_t}(u_i)\Bigr\rangle_{\nX},
\label{GDalgor}
\end{align}
where $\alpha>0$ is the step size and
\[
\langle f,g\rangle_{\nX} := \frac{1}{\nX}\sum_{k=1}^{\nX} f(x_k)g(x_k)
\]
denotes the empirical inner product over the second-stage samples.

\cite{nguyen2024optimalconvergenceratesneural} employ a symmetric initialization scheme for the network parameters $\theta_{0}$ to ensure that $G_{\theta_{0}} \equiv 0$. Importantly, this symmetric trick does not affect the limiting neural tangent kernel (NTK); see \citet{zhang2020type} for details. 

Specifically, the weights in the output layer are initialized symmetrically as
\[
a_{m}^{(0)} = \tau \quad \text{for } m = 1, \dots, M/2,
\qquad 
a_{m}^{(0)} = -\tau \quad \text{for } m = M/2+1, \dots, M,
\]
where $\tau > 0$ is a fixed constant. The input layer parameters are initialized in a coupled manner,
\[
b_{m}^{(0)} = b_{m+M/2}^{(0)} 
\quad \text{for } m \in \{1, \ldots, M/2\},
\]
where the first half of the parameters $\{b_{m}^{(0)}\}_{m=1}^{M/2}$ are drawn independently from the initialization distribution $\pi_{0}$.

Now we are ready to state the original theorem from \citet{nguyen2024optimalconvergenceratesneural}, which provides a generalization bound for Neural Operators trained via the empirical GD algorithm.

\begin{theorem}[\citet{nguyen2024optimalconvergenceratesneural}, Theorem~3.5]
Suppose Assumptions~\ref{ass:input}, \ref{ass:source}, and \ref{ass:dim} hold.  
Let $G_{\theta_{T_n}}$ denote the Neural Operator as defined in Section~\ref{sec:motivation}, where $T_n$ is the number of GD iterations in~\eqref{GDalgor}.  
Let $M_{\nU}$ denote the network width.  
Assume that $\alpha \in (0,\kappa^{-2})$, $\nU \ge n_0 := e^{\frac{2r+b}{2r+b-1}}$, and
\[
T_{\nU} = C\nU^{\frac{1}{2r+b}}, \qquad 
M_{\nU} \ge \tilde{C} B_{T_{\nU}}^6 \log^2(\nU) T_{\nU}^{2r \vee 1}, \qquad
\nX \ge \tilde{C} B_{T_{\nU}}^2 T_{\nU}^{2r} \log^2 T_{\nU},
\]
with
\begin{equation}\label{weights}
   \|\theta_t - \theta_0\|_\Theta \le B_{T_{\nU}} \quad \text{for all } t \in [T_{\nU}].
\end{equation}
Then, with probability at least $1-\delta$,
\begin{align}
\| G_{\theta_{T_{\nU}}} - G_\rho \|_{L^2(\rho_\mathcal{U})} 
\;\le\; \bar{C} \,\nU^{-\frac{r}{2r+b}} \log^3 \frac{2}{\delta},
\label{eq:final-bound}
\end{align}
where $C,\tilde{C},\bar{C}>0$ are independent of $\nU,\nX,M_{\nU},T_{\nU},B_{T_{\nU}}$.
\end{theorem}

\paragraph{Discussion.}
\citet[Theorem~3.7]{nguyen2024optimalconvergenceratesneural} show that \eqref{weights} holds with high probability for $B_T = O(\log T)$.  
Consequently, a two-layer Neural Operator trained by gradient descent achieves the minimax-optimal learning rate $\nU^{-\frac{r}{2r+b}}$.  
Notably, the required network width $M$ matches the number of random features needed for kernel gradient descent and is independent of $\nX$.  
The lower bound on $\nX$ arises from controlling the first term in the error decomposition~\eqref{NOdecomp}, which accounts for the discretization of function-valued samples.  
To ensure that the empirical gradient descent in~\eqref{GDalgor} closely tracks the population gradient flow, the empirical inner product over second-stage samples must uniformly approximate the $L^2(\mathcal{X},\mu)$ inner product. 
Hoeffding-type concentration inequalities yield an $O(\nX^{-1/2})$ discretization error, which must match the minimax-optimal learning rate.  
This requirement leads to the stated lower bound on $\nX$.

%% file: AppB.tex
\section{Proofs}

In this section, we provide the proofs of our main results.

\paragraph{Notation.} 
Throughout the proofs, we use the following shorthand notation. For any bounded linear operator mapping between two Hilbert spaces $A$ and $\lambda > 0$, we write $A_\lambda := A + \lambda I$, where $I$ denotes the identity operator. For $G : \mathcal{U} \to \mathcal{V}$ and $(u_1,...,u_n) \in \cU^n$, we define the vector $\bar{G} := \bigl(G(u_1), \dots, G(u_n)\bigr) \in \mathcal{V}^n$. 
Furthermore, let $C_\bullet > 0$ denote a generic constant that may change from line to line but depends only on the quantities $\kappa, r, b, c_q, c_b, E, D, Q, Z$, and not on $\delta, p, \lambda,$ or $n$.

\medskip
\noindent
We recall the following operator definitions. Let $\mathcal{S}_M : \cH_M \hookrightarrow L^2(\mathcal{U}, \rho_\mathcal{U})$ denote the inclusion operator of $\cH_M$ into $L^2(\mathcal{U}, \rho_\mathcal{U})$ for $M \in \mathbb{N}$. Its adjoint $\cS^{*}_M: L^{2}(\mathcal{U}, \rho_\mathcal{U}) \to \mathcal{H}_{M}$ is given by
\[
\cS^{*}_M G = \int_{\mathcal{U}}  K_{M,u} \, G(u) \, \rho_\mathcal{U}(\mathrm{d}u).
\]

\noindent
The covariance operator $\Sigma_M: \mathcal{H}_{M} \to \mathcal{H}_{M}$ and the kernel integral operator $\mathcal{L}_M: L^2(\mathcal{U}, \rho_\mathcal{U}) \to L^2(\mathcal{U}, \rho_\mathcal{U})$ are defined as
\begin{align*}
   \Sigma_M G &:= \int_{\mathcal{U}} K_{M,u} K_{M,u}^* G \, \rho_\mathcal{U}(\mathrm{d}u), \\
   \mathcal{L}_M G &:= \int_{\mathcal{U}} K_{M,u} G(u) \, \rho_\mathcal{U}(\mathrm{d}u).
\end{align*}

\noindent
The empirical counterparts of these operators, obtained by replacing $\rho_\mathcal{U}$ with the empirical measure, are given by
\begin{align*}
\widehat{\cS}_{M}: \mathcal{H}_{M} \to \mathcal{V}^{n}, 
&\qquad (\widehat{\cS}_{M} G)_{j} = K_{M,u_{j}}^* G, \\
\widehat{\cS}_{M}^{*}: \mathcal{V}^{n} \to \mathcal{H}_{M}, 
&\qquad \widehat{\cS}_{M}^{*} \mathbf{v} = \frac{1}{n} \sum_{j=1}^{n} K_{M,u_{j}} v_{j}, \\
\widehat{\Sigma}_{M}: \mathcal{H}_{M} \to \mathcal{H}_{M}, 
&\qquad \widehat{\Sigma}_{M} = \frac{1}{n} \sum_{j=1}^{n} K_{M,u_{j}} K_{M,u_{j}}^*.\\
\end{align*}


\subsection{Proof Organization and Error Decomposition}

In the classical kernel setting without random features, the standard analysis introduces the idealized population estimator
\[
F_\lambda^* \coloneqq \mathcal{S}^*  \, \phi_\lambda(\mathcal{L}) \, G_{\rho},
\]
which enables a decomposition of the error into bias and variance components:
\begin{align*}
G_\rho - \mathcal{S} \widehat{F}_\lambda 
&= \bigl(G_\rho - \mathcal{S} F_\lambda^*\bigr) 
   + \bigl(\mathcal{S} F_\lambda^* - \mathcal{S} \widehat{F}_\lambda\bigr) \\
&= r_\lambda(\mathcal{L})\,\mathcal{L}^r H 
   + \bigl(\mathcal{S} F_\lambda^* - \mathcal{S} \widehat{F}_\lambda\bigr),
\end{align*}
where the bias is controlled via the residual polynomial \( r_\lambda \), while the variance is handled through Hoeffding-type concentration inequalities \citep{Muecke2017op.rates}.

\medskip
\noindent
In the random feature setting, our analysis adapts this approach by introducing the idealized estimator
\[
F_\lambda^* \coloneqq \mathcal{S}^*_M \, \phi_\lambda(\mathcal{L}_M) \, G_{\rho},
\]
which naturally leads to an additional \emph{approximation error} due to the finite-dimensional random feature space. Specifically, we decompose the excess risk as
\begin{align}
\|G_\rho - \mathcal{S}_M F_\lambda^M\|_{L^2(\rho_\mathcal{U})}
&\leq \|G_\rho - \mathcal{S}_M F_\lambda^*\|_{L^2(\rho_\mathcal{U})} 
  + \|\mathcal{S}_M F_\lambda^* - \mathcal{S}_M F_\lambda^M\|_{L^2(\rho_\mathcal{U})}\nonumber\\[4pt]
&=: \text{Approximation Error} + \text{Estimation Error}. 
\label{excessrisk2}
\end{align}

\medskip
\noindent
A key technical challenge lies in controlling the approximation error, which requires comparing the population operator \( \mathcal{L}^r \) with its random feature counterpart \( \mathcal{L}_M^r \). While prior work \citep{features} focuses on the case \( r \in [0.5, 1] \), we develop in Section~\ref{techineq} novel operator inequalities that allow us to treat arbitrary \( r > 0 \).
For the estimation term, classical analyses rely on the specific structure of kernel ridge regression. In contrast, our approach uses a refined decomposition that exploits the polynomial structure of the residual \( r_\lambda \), enabling us to obtain variance bounds \emph{uniformly over all regularization filters}.

\medskip
\noindent
We bound the approximation and estimation errors separately in Sections~\ref{I} and~\ref{II}, respectively, and then combine these bounds in Section~\ref{III} to prove our main result, Theorem~\ref{theo1}. The required operator inequalities are deferred to Section~\ref{techineq}, while the necessary concentration inequalities are collected in Section~\ref{concineq}.


\subsection{Bounding the Approximation Error}
\label{I}

\begin{proposition}
\label{mainprop2}
Suppose that Assumptions~\ref{ass:input}, \ref{ass:kernel}, \ref{ass:source}, ~\ref{ass:dim} and $\nu\geq r\vee1$ hold.  
For any $\lambda \in (0,1]$, assume that
\begin{align*}
M \;\ge\; p \, C_\bullet \, \log^2\!\bigl(\delta^{-1}\bigr)\, \log\!\bigl(\lambda^{-1}\bigr) \cdot
\begin{cases}
\lambda^{-1}, & r \in \bigl(0,\tfrac{1}{2}\bigr),\\[4pt]
\lambda^{\,b(1-2r)-1}, & r \in \bigl[\tfrac{1}{2},1\bigr],\\[4pt]
\lambda^{-2r}, & r \in (1,\infty),
\end{cases}
\end{align*}
then the approximation term in~\eqref{excessrisk2} satisfies, with probability at least $1-\delta$,
\begin{align*}
\|G_\rho - \mathcal{S}_M F_\lambda^*\|_{L^2(\rho_\mathcal{U})}
\;\le\; C_\bullet \, \lambda^{r}.
\end{align*}
\end{proposition}

\begin{proof}
From Assumption~\ref{ass:source}, we have $G_\rho = \mathcal{L}^r H$ with $\|H\|_{L^2(\rho_\mathcal{U})} \le R$. Hence,
\begin{align}
\|G_\rho - \mathcal{S}_M F_\lambda^*\|_{L^2(\rho_\mathcal{U})} 
&= \bigl\|\bigl(\mathcal{L}_M \phi_\lambda(\mathcal{L}_M) - I\bigr)\mathcal{L}^r H \bigr\|_{L^2(\rho_\mathcal{U})} \nonumber\\[4pt]
&\le R \bigl\| r_\lambda(\mathcal{L}_M)\, \mathcal{L}^r \bigr\|,
\label{eq:main-bias-step}
\end{align}
where $r_\lambda$ denotes the residual polynomial defined in~\eqref{residual}.

\medskip
\noindent
For the remaining term, we obtain
\begin{align*}
R \bigl\| r_\lambda(\mathcal{L}_M)\, \mathcal{L}^r \bigr\|
&\le R \bigl\| r_\lambda(\mathcal{L}_M)\, \mathcal{L}_{M,\lambda}^{(r \vee 1)} \bigr\| 
     \bigl\| \mathcal{L}_{M,\lambda}^{-(r \vee 1)} \mathcal{L}_{\lambda}^r \bigr\| \\[4pt]
&\le 3 R\, c_{r \vee 1}\, \lambda^{r},
\end{align*}
where the last inequality follows from the bounds 
$\|r_\lambda(\mathcal{L}_M)\mathcal{L}_{M,\lambda}^{(r \vee 1)}\| 
\le c_{r \vee 1} \lambda^{(r \vee 1)}$ from~\eqref{c_r}  
and $\|\mathcal{L}_{M,\lambda}^{-(r \vee 1)} \mathcal{L}_{\lambda}^r\| 
\le 3 \lambda^{-(1-r)^+}$ from Proposition~\ref{OPbound7},  
which holds with probability at least $1 - 3\delta$.
\end{proof}


\subsection{Bounding the Estimation Error}
\label{II}

We now turn to bounding the variance-type error in~\eqref{excessrisk2}. 
To this end, we decompose it into two parts: 
(i) a classical estimation error term, which can be controlled via standard concentration inequalities, and 
(ii) an additional approximation-type term, whose contribution is bounded using properties of the residual polynomial.

\begin{proposition}
\label{mainprop}
Suppose that Assumptions~\ref{ass:input}, \ref{ass:kernel}, \ref{ass:source}, and~\ref{ass:dim} hold, 
and let $\nu \ge r \vee 1$. Then, for any $s \in [0, \tfrac{1}{2}]$ and $\lambda \in (0,1]$, the following holds with probability at least $1-\delta$:
\begin{align*}
\left\|\Sigma_M^{\frac{1}{2}-s} (F_\lambda^M - F_\lambda^*)\right\|_{\mathcal{H}_M}
\;\le\; C_\bullet \log\!\frac{1}{\delta}\, \lambda^{r-s},
\end{align*}
provided that
\begin{align*}
M \;\ge\; p \, C_\bullet \, \log^2\!\bigl(\delta^{-1}\bigr)\, \log\!\bigl(\lambda^{-1}\bigr) \cdot
\begin{cases}
\lambda^{-1}, & r \in \bigl(0,\tfrac{1}{2}\bigr),\\[4pt]
\lambda^{\,b(1-2r)-1}, & r \in \bigl[\tfrac{1}{2},1\bigr],\\[4pt]
\lambda^{-2r}, & r \in (1,\infty),
\end{cases}
\end{align*}
and
\begin{align*}
n \;\ge\; C_\bullet\,\log^{3(2r+b)}\bigl(\delta^{-1}\bigr)\,\lambda^{-(2r+b)}\,, \quad n \geq n_0 := \exp\!\left(\tfrac{2r+b}{2r+b-1}\right).
\end{align*}
\end{proposition}

\begin{proof}
We begin with the decomposition
\begin{align}
\bigl\|\Sigma_M^{\frac{1}{2}-s} (F_\lambda^M - F_\lambda^*)\bigr\|_{\mathcal{H}_M} 
&\le \bigl\|\Sigma_M^{\frac{1}{2}-s}\bigl(\phi_\lambda(\widehat\Sigma_M)\widehat{\mathcal{S}}_M^{*}\mathbf{v} 
 - \phi_\lambda(\widehat\Sigma_M)\widehat{\Sigma}_M F_\lambda^*\bigr)\bigr\|_{\mathcal{H}_M} \nonumber\\[2pt]
&\quad + \bigl\|\Sigma_M^{\frac{1}{2}-s}\bigl(\phi_\lambda(\widehat\Sigma_M)\widehat{\Sigma}_M - I\bigr)F_\lambda^*\bigr\|_{\mathcal{H}_M} \nonumber\\[4pt]
&= \bigl\|\Sigma_M^{\frac{1}{2}-s}\phi_\lambda(\widehat\Sigma_M)\widehat{\mathcal{S}}_M^{*}\bigl(\mathbf{v} - \widehat{\mathcal{S}}_M F_\lambda^*\bigr)\bigr\|_{\mathcal{H}_M}
  + \bigl\|\Sigma_M^{\frac{1}{2}-s}r_\lambda(\widehat\Sigma_M)F_\lambda^*\bigr\|_{\mathcal{H}_M} \nonumber\\[2pt]
&=: \textbf{(I)} + \textbf{(II)}. 
\end{align}

We bound the two terms $\textbf{(I)}$ and $\textbf{(II)}$ separately. 
Specifically, by Proposition~\ref{T2I}, with probability at least $1 - \delta$,
\[
\textbf{(I)} \;\le\; C_\bullet\,\log\!\tfrac{1}{\delta}\!\,\,\lambda^{r-s},
\]
and 
\[
\textbf{(II)} \;\le\; C_\bullet\,\, \lambda^{r-s}.
\]
Combining these bounds yields the stated result.
\end{proof}

\begin{proposition}
\label{T2I} 
Suppose the assumptions of Proposition~\ref{mainprop} hold. 
Then, for any $s \in [0,\tfrac{1}{2}]$ and $\lambda \in (0,1]$, with probability at least $1-\delta$,
\begin{align}
\textbf{(I)}\quad 
&\bigl\|\Sigma_M^{\frac{1}{2}-s}\,\phi_\lambda(\widehat\Sigma_M)\,\widehat{\mathcal{S}}_{M}^{*}\bigl(\mathbf{v}-\widehat{\mathcal{S}}_M F_\lambda^*\bigr)\bigr\|_{\mathcal{H}_M}
\;\le\; C_\bullet\,\log\!\tfrac{1}{\delta}\,\lambda^{r-s},
\label{a)techniqineq}
\\[5pt]
\textbf{(II)}\quad 
&\bigl\|\Sigma_M^{\frac{1}{2}-s}\,r_\lambda(\widehat\Sigma_M)\,F_\lambda^*\bigr\|_{\mathcal{H}_M}
\;\le\; C_\bullet\,\lambda^{r-s}.
\label{b)techniqineq}
\end{align}
\end{proposition}

\begin{proof}
\textbf{(I)} We begin with the decomposition
\begin{align}
\nonumber
\bigl\|\Sigma_M^{\frac{1}{2}-s}\,\phi_\lambda(\widehat\Sigma_M)\,\widehat{\mathcal{S}}_{M}^{*}\bigl(\mathbf{v}-\widehat{\mathcal{S}}_M F_\lambda^*\bigr)\bigr\|_{\mathcal{H}_M}
&\le 
\bigl\|\Sigma_M^{\frac{1}{2}-s}\phi_\lambda(\widehat\Sigma_M)\Sigma_{M,\lambda}^{\frac{1}{2}}\bigr\|\,
\bigl\|\Sigma_{M,\lambda}^{-\frac{1}{2}}\widehat{\mathcal{S}}_{M}^{*}\bigl(\mathbf{v}-\widehat{\mathcal{S}}_M F_\lambda^*\bigr)\bigr\|_{\mathcal{H}_M}
\\
&=: i \cdot ii .
\label{I.II}
\end{align}

\paragraph{Step (i).}
By Proposition~\ref{OPbound6}, with probability at least $1-4\delta$,
\begin{align}\label{usebound6}
\|\widehat{\Sigma}_{M,\lambda}^{-\frac{1}{2}}\Sigma_{M,\lambda}^{\frac{1}{2}}\|\le 2.
\end{align}
Hence
\begin{align*}
\bigl\|\Sigma_M^{\frac{1}{2}-s}\phi_\lambda(\widehat\Sigma_M)\Sigma_{M,\lambda}^{\frac{1}{2}}\bigr\|
&\le \lambda^{-s}\bigl\|\Sigma_{M,\lambda}^{\frac{1}{2}}\phi_\lambda(\widehat\Sigma_M)\Sigma_{M,\lambda}^{\frac{1}{2}}\bigr\|\\
&\le \lambda^{-s}\bigl\|\widehat\Sigma_{M,\lambda}\phi_\lambda(\widehat\Sigma_M)\bigr\|\,
\bigl\|\Sigma_{M,\lambda}^{\frac{1}{2}}\widehat{\Sigma}_{M,\lambda}^{-\frac{1}{2}}\bigr\|^2
\;\le\; 4D\,\lambda^{-s},
\end{align*}
where $D$ is defined in~\eqref{def.phi}.

\paragraph{Step (ii).}
We decompose
\begin{align*}
\bigl\|\Sigma_{M,\lambda}^{-\frac{1}{2}}\widehat{\mathcal{S}}_{M}^{*}\bigl(\mathbf{v}-\widehat{\mathcal{S}}_M F_\lambda^*\bigr)\bigr\|_{\mathcal{H}_M}
&\le 
\bigl\|\Sigma_{M,\lambda}^{-\frac{1}{2}}\widehat{\mathcal{S}}_{M}^{*}\bigl(\mathbf{v}-\bar{G}_\rho\bigr)\bigr\|_{\mathcal{H}_M}
+
\bigl\|\Sigma_{M,\lambda}^{-\frac{1}{2}}\widehat{\mathcal{S}}_{M}^{*}\bigl(\bar{G}_\rho-\widehat{\mathcal{S}}_M F_\lambda^*\bigr)\bigr\|_{\mathcal{H}_M}\\
&=: a + b .
\end{align*}

\textit{Term (a):}
Using Proposition~\ref{concentrationineq1} together with Proposition~\ref{prop:effecdim2}, we obtain, with probability at least $1-3\delta$,

\begin{align*}
\bigl\|\Sigma_{M,\lambda}^{-\frac{1}{2}}\widehat{\mathcal{S}}_{M}^{*}\bigl(\mathbf{v}-\bar{G}_\rho\bigr)\bigr\|_{\mathcal{H}_M}
&\le 
\left(\frac{4QZ\kappa}{\sqrt{\lambda}n}+\frac{8Q\sqrt{(1+2\log\frac{2}{\delta})\mathcal{N}_{\mathcal{L}}(\lambda)}}{\sqrt{n}}\right)
\log\tfrac{2}{\delta}\\
&\le 
C_\bullet
\left(\frac{1}{\sqrt{\lambda}n}+\sqrt{\frac{\log\tfrac{1}{\delta}}{n\lambda^b}}\right)
\log\tfrac{1}{\delta}
\;\le\;
C_\bullet\,\lambda^{r}\log\tfrac{1}{\delta},
\end{align*}
where the last inequality follows from the assumption on $n$.

\textit{Term (b):}
From Proposition~\ref{OPbound6},
\[
\bigl\|\Sigma_{M,\lambda}^{-\frac{1}{2}}\widehat{\mathcal{S}}_{M}^{*}\bigr\|^2\leq 2\|\widehat{\Sigma}_{M,\lambda}^{-\frac{1}{2}}\widehat{\mathcal{S}}_{M}^{*}\bigr\|^2
=2\bigl\|\widehat{\mathcal{S}}_M^*\widehat{\mathcal{S}}_M(\widehat{\mathcal{S}}_M^*\widehat{\mathcal{S}}_M+\lambda)^{-1}\bigr\|\le 2.
\]
Therefor we obtain by Proposition~\ref{concentrationineq2}, with probability at least $1-\delta$,
\begin{align*}
\bigl\|\Sigma_{M,\lambda}^{-\frac{1}{2}}\widehat{\mathcal{S}}_{M}^{*}&(\bar{G}_\rho-\widehat{\mathcal{S}}_M F_\lambda^*)\bigr\|_{\mathcal{H}_M}
\leq \frac{2}{\sqrt{n}}\|\bar{G}_\rho-\widehat{\mathcal{S}}_M F_\lambda^*\|_2\\
&\leq 2\sqrt{\left|\frac{1}{n}\|\bar{G}_\rho-\widehat{\mathcal{S}}_M F_\lambda^*\|_2^2-\|G_\rho-\mathcal{S}_M F_\lambda^*\|_{L^2(\rho_\mathcal{U})}^2\right|}
+\|G_\rho-\mathcal{S}_M F_\lambda^*\|_{L^2(\rho_\mathcal{U})}\\
&\leq 
C_\bullet\sqrt{\left(\frac{\lambda^{-2(\frac{1}{2}-r)^+}}{n}+\frac{\lambda^{-(\frac{1}{2}-r)^+}\|G_\rho-\mathcal{S}_M F_\lambda^*\|_{L^2(\rho_\mathcal{U})}}{\sqrt{n}}\right)\log\tfrac{1}{\delta}}
+\|G_\rho-\mathcal{S}_M F_\lambda^*\|_{L^2(\rho_\mathcal{U})}.
\end{align*}

Applying Proposition~\ref{mainprop2} yields, with probability at least $1-3\delta$,
\begin{align*}
\bigl\|\Sigma_{M,\lambda}^{-\frac{1}{2}}\widehat{\mathcal{S}}_{M}^{*}(\bar{G}_\rho-\widehat{\mathcal{S}}_M F_\lambda^*)\bigr\|_{\mathcal{H}_M}
&\le 
C_\bullet\sqrt{\left(\frac{\lambda^{-(1-2r)^+}}{n}+\frac{\lambda^{-(\frac{1}{2}-r)^+}\lambda^{r}}{\sqrt{n}}+\lambda^{2r}\right)\log\tfrac{1}{\delta}}\\
&\le 
C_\bullet\,\lambda^{r}\sqrt{\log\tfrac{1}{\delta}},
\end{align*}
where the last inequality follows from the assumption on $n$.
Therefore,
\[
ii \le C_\bullet\,\lambda^{r}\log\tfrac{1}{\delta}.
\]

Combining $(i)$ and $(ii)$ in~\eqref{I.II} proves~\eqref{a)techniqineq}.
Collecting all probabilities and applying Proposition~\ref{conditioning} gives total probability at least $1-11\delta$.
Redefining $\delta$ completes part~(I).


\medskip
\noindent
\textbf{(II)}
We next bound the term in~\eqref{b)techniqineq}. With probability at least $1-4\delta$,
\begin{align*}
\|\Sigma_M^{\frac{1}{2}-s}r_\lambda(\widehat\Sigma_M)F_\lambda^*\|_{\mathcal{H}_M}
&\le \lambda^{-s}\|\Sigma_{M,\lambda}^{\frac{1}{2}}r_\lambda(\widehat\Sigma_M)F_\lambda^*\|_{\mathcal{H}_M}\\
&\le \lambda^{-s}\|\Sigma_{M,\lambda}^{\frac{1}{2}}\widehat\Sigma_{M,\lambda}^{-\frac{1}{2}}\|\,
\|\widehat\Sigma_{M,\lambda}^{\frac{1}{2}}r_\lambda(\widehat\Sigma_M)F_\lambda^*\|_{\mathcal{H}_M}\\
&\le 2\lambda^{-s}\|\widehat\Sigma_{M,\lambda}^{\frac{1}{2}}r_\lambda(\widehat\Sigma_M)F_\lambda^*\|_{\mathcal{H}_M},
\end{align*}
where we again used Proposition~\ref{OPbound6}.
Writing out $F_\lambda^*=\mathcal{S}_M^*\phi_\lambda(\mathcal{L}_M)\mathcal{L}^r H$ gives
\begin{align}
2\lambda^{-s}\|\widehat\Sigma_{M,\lambda}^{\frac{1}{2}}r_\lambda(\widehat\Sigma_M)F_\lambda^*\|_{\mathcal{H}_M}
\;\le\;
2R\lambda^{-s}\|\widehat\Sigma_{M,\lambda}^{\frac{1}{2}}r_\lambda(\widehat\Sigma_M)\mathcal{S}_M^*\phi_\lambda(\mathcal{L}_M)\mathcal{L}^r\|.
\label{casesT2}
\end{align}

We distinguish two cases.

\paragraph{Case $r \le \tfrac{1}{2}$.}
We have
\begin{align*}
\|\widehat\Sigma_{M,\lambda}^{\frac{1}{2}}r_\lambda(\widehat\Sigma_M)\mathcal{S}_M^*\phi_\lambda(\mathcal{L}_M)\mathcal{L}^r\|
&\le 
\|\widehat\Sigma_{M,\lambda}^{\frac{1}{2}}r_\lambda(\widehat\Sigma_M)\mathcal{S}_M^*\phi_\lambda(\mathcal{L}_M)\mathcal{L}_{M,\lambda}^r\|
\|\mathcal{L}_{M,\lambda}^{-r}\mathcal{L}_{\lambda}^r\|\\
&= 
\|\widehat\Sigma_{M,\lambda}^{\frac{1}{2}}r_\lambda(\widehat\Sigma_M)\Sigma_{M,\lambda}^r\mathcal{S}_M^*\phi_\lambda(\mathcal{L}_M)\|
\|\mathcal{L}_{M,\lambda}^{-r}\mathcal{L}_{\lambda}^r\|\\
&\le 
\|\widehat\Sigma_{M,\lambda}^{\frac{1}{2}}r_\lambda(\widehat\Sigma_M)\Sigma_{M,\lambda}^r\|
\|\mathcal{L}_M^{\frac{1}{2}}\phi_\lambda(\mathcal{L}_M)\|
\|\mathcal{L}_{M,\lambda}^{-r}\mathcal{L}_{\lambda}^r\|.
\end{align*}
By Proposition~\ref{ineqvolkan}, $\|\mathcal{L}_M^{\frac{1}{2}}\phi_\lambda(\mathcal{L}_M)\|\le D\lambda^{-1/2}$.
Moreover, by Proposition~\ref{ineq2} and Proposition~\ref{OPbound5}, $\|\mathcal{L}_{M,\lambda}^{-r}\mathcal{L}_{\lambda}^r\|\le 2$ with probability at least $1-4\delta$.
Hence,
\begin{align}
\|\Sigma_M^{\frac{1}{2}-s}r_\lambda(\widehat\Sigma_M)F_\lambda^*\|_{\mathcal{H}_M}
\le 4DR\lambda^{-s-\frac{1}{2}}\|\widehat\Sigma_{M,\lambda}^{\frac{1}{2}}r_\lambda(\widehat\Sigma_M)\Sigma_{M,\lambda}^r\|.
\label{TIIcase1}
\end{align}
To bound the remaining term, use Proposition~\ref{OPbound6}:
$\|\widehat{\Sigma}_{M,\lambda}^{-r}\Sigma_{M,\lambda}^r\|\le 2$. 
Together with~\eqref{c_r}, this yields
\[
\|\widehat\Sigma_{M,\lambda}^{\frac{1}{2}}r_\lambda(\widehat\Sigma_M)\Sigma_{M,\lambda}^r\|
\le 2c_{\frac{1}{2}+r}\lambda^{\frac{1}{2}+r}.
\]
Plugging this into~\eqref{TIIcase1} gives
\[
\|\Sigma_M^{\frac{1}{2}-s}r_\lambda(\widehat\Sigma_M)F_\lambda^*\|_{\mathcal{H}_M}
\le 8DR\,c_{\frac{1}{2}+r}\lambda^{r-s}.
\]

\paragraph{Case $r > \tfrac{1}{2}$.}
We proceed analogously:
\begin{align*}
\|\widehat\Sigma_{M,\lambda}^{\frac{1}{2}}r_\lambda(\widehat\Sigma_M)\mathcal{S}_M^*\phi_\lambda(\mathcal{L}_M)\mathcal{L}^r\|
&\le
\|\widehat\Sigma_{M,\lambda}^{\frac{1}{2}}r_\lambda(\widehat\Sigma_M)\mathcal{S}_M^*\phi_\lambda(\mathcal{L}_M)\mathcal{L}_{M,\lambda}^{(r\vee1)}\|
\|\mathcal{L}_{M,\lambda}^{-(r\vee1)}\mathcal{L}_{\lambda}^r\|\\
&\le 
\|\widehat\Sigma_{M,\lambda}^{\frac{1}{2}}r_\lambda(\widehat\Sigma_M)\Sigma_{M,\lambda}^{(r\vee1)-\frac{1}{2}}\|
\|\mathcal{L}_M\phi_\lambda(\mathcal{L}_M)\|
\|\mathcal{L}_{M,\lambda}^{-(r\vee1)}\mathcal{L}_{\lambda}^r\|.
\end{align*}
By the spectral method properties and Proposition~\ref{OPbound7}, $\|\mathcal{L}_M\phi_\lambda(\mathcal{L}_M)\|\le D$ and $\|\mathcal{L}_{M,\lambda}^{-(r\vee1)}\mathcal{L}_{\lambda}^r\|\le 3\lambda^{-(1-r)^+}$. 
Hence,
\begin{align}
\|\Sigma_M^{\frac{1}{2}-s}r_\lambda(\widehat\Sigma_M)F_\lambda^*\|_{\mathcal{H}_M}
\le \frac{6DR}{\lambda^{s+(1-r)^+}}\|\widehat\Sigma_{M,\lambda}^{\frac{1}{2}}r_\lambda(\widehat\Sigma_M)\Sigma_{M,\lambda}^{(r\vee1)-\frac{1}{2}}\|.
\label{TIIcase2}
\end{align}
Using~\eqref{c_r} and Proposition~\ref{OPbound8}, with probability at least $1-\delta$,
\begin{align}\label{usebound8}
\|\widehat\Sigma_{M,\lambda}^{\frac{1}{2}}r_\lambda(\widehat\Sigma_M)\Sigma_{M,\lambda}^{(r\vee1)-\frac{1}{2}}\|
\le 2c_{r\vee1}\lambda^{(r\vee1)}.
\end{align}
Substituting this into~\eqref{TIIcase2} yields
\[
\|\Sigma_M^{\frac{1}{2}-s}r_\lambda(\widehat\Sigma_M)F_\lambda^*\|_{\mathcal{H}_M}
\le 12DR\,c_{\frac{1}{2}+r}\lambda^{r-s}.
\]

Combining both cases establishes~\eqref{b)techniqineq}.
Collecting all concentration bounds and applying Proposition~\ref{conditioning} gives probability at least $1-8\delta$. Redefining $\delta$ completes the proof.
\end{proof}


\subsection{Combining the Error Bounds}
\label{III}

\begin{theorem}
\label{theo2}
Suppose that Assumptions~\ref{ass:input}, \ref{ass:kernel}, \ref{ass:source}, and~\ref{ass:dim} hold, 
and let $\nu \ge r \vee 1$. Then, for any $s \in [0, \tfrac{1}{2}]$ and $\lambda \in (0,1]$, the following holds with probability at least $1-\delta$:
\begin{align*}
\|G_\rho - \mathcal{S}_M F_\lambda^M\|_{L^2(\rho_\mathcal{U})}
\;\le\;
C_\bullet\,\log\!\tfrac{1}{\delta}\,\lambda^{r},
\end{align*}
provided that
\begin{align*}
M \;\ge\; p\,C_\bullet\,\log^2\!\bigl(\delta^{-1}\bigr)\,\log\!\bigl(\lambda^{-1}\bigr)\!\cdot
\begin{cases}
\lambda^{-1}, & r \in \bigl(0,\tfrac{1}{2}\bigr),\\[4pt]
\lambda^{\,b(1-2r)-1}, & r \in \bigl[\tfrac{1}{2},1\bigr],\\[4pt]
\lambda^{-2r}, & r \in (1,\infty),
\end{cases}
\end{align*}
and
\begin{align*}
n \;\ge\; C_\bullet\,\log^{3(2r+b)}\bigl(\delta^{-1}\bigr)\,\lambda^{-(2r+b)}\,, \quad n \geq n_0 := \exp\!\left(\tfrac{2r+b}{2r+b-1}\right).
\end{align*}
\end{theorem}

\begin{proof}
We begin with the following decomposition:
\begin{align}
\|G_\rho - \mathcal{S}_M F_\lambda^M\|_{L^2(\rho_\mathcal{U})}
\;\le\;
\|G_\rho - \mathcal{S}_M F_\lambda^*\|_{L^2(\rho_\mathcal{U})}
\;+\;
\|\mathcal{S}_M (F_\lambda^M - F_\lambda^*)\|_{L^2(\rho_\mathcal{U})}
\;=:\;
T_1 + T_2.
\label{Maindecomposition}
\end{align}

We now bound $T_1$ and $T_2$ separately.

\paragraph{Step 1: Bounding $T_1$.}
By Proposition~\ref{mainprop2}, with probability at least $1-\delta$,
\begin{align}
\|G_\rho - \mathcal{S}_M F_\lambda^*\|_{L^2(\rho_\mathcal{U})}
\;\le\;
C_\bullet\,\lambda^{r}.
\label{(T_1)}
\end{align}

\paragraph{Step 2: Bounding $T_2$.}
Using Mercer’s theorem (see, e.g.,~\cite{Ingo}) and Proposition~\ref{mainprop}, we have, with probability at least $1-\delta$,
\begin{align*}
\|\mathcal{S}_M (F_\lambda^M - F_\lambda^*)\|_{L^2(\rho_\mathcal{U})}
\;=\;
\|\Sigma_M^{\frac{1}{2}} (F_\lambda^M - F_\lambda^*)\|_{\mathcal{H}_M}
\;\le\;
C_\bullet\,\log\!\tfrac{1}{\delta}\,\lambda^{r}.
\end{align*}

\paragraph{Conclusion.}
Combining the bounds for $T_1$ and $T_2$ in~\eqref{Maindecomposition} establishes the result:
\[
\|G_\rho - \mathcal{S}_M F_\lambda^M\|_{L^2(\rho_\mathcal{U})}
\;\le\;
C_\bullet\,\log\!\tfrac{1}{\delta}\,\lambda^{r}.
\]
\end{proof}

We finally note that Theorem~\ref{theo1} follows directly from Theorem~\ref{theo2}.

%% file: T.U.tex

\subsection{Technical Inequalities} 
\label{techineq}

\begin{proposition}
\label{conditioning}
Let  $E_i$ be events with probability at least $1-\delta_i$ and set
$$
E:=\bigcap^k_{i=1}E_i.
$$
If we can show for some event $A$ that $\mathbb{P}(A|E)\geq 1-\delta$  then we also have
\begin{align*}
\mathbb{P}(A)&\geq\int_{E}\mathbb{P}(A|\omega)d\mathbb{P}(\omega)
\geq (1-\delta)\mathbb{P}(E)\\
&=(1-\delta)\left(1-\mathbb{P}\left(\bigcup_{i=1}^k (\Omega/E_i)\right)\right)\geq(1-\delta)\left(1-\sum_{i=1}^k\delta_i\right).
\end{align*}
\end{proposition}

\begin{proposition}[ \cite{aleksandrov2009operatorholderzygmundfunctions}, \cite{Muecke2017op.rates} (Proposition B.1.) ]
\label{ineq1}
Let $B_{1}, B_{2}$ be two non-negative self-adjoint operators on some Hilbert space with $\left\|B_{j}\right\| \leq a, j=1,2$, for some non-negative a.
\begin{itemize}
\item[(i)] If $0 \leq r \leq 1$, then
$$
\left\|B_{1}^{r}-B_{2}^{r}\right\| \leq C_{r}\left\|B_{1}-B_{2}\right\|^{r},
$$
for some $C_{r}<\infty$.
\item[(ii)] If $r>1$, then
$$
\left\|B_{1}^{r}-B_{2}^{r}\right\| \leq C_{a, r}\left\|B_{1}-B_{2}\right\|,
$$
for some $C_{a, r}<\infty$. 
\end{itemize}
\end{proposition}

\begin{proposition}[Fujii et al., 1993, Cordes inequality]
\label{ineq2}
Let $A$ and $B$ be two positive bounded linear operators on a separable Hilbert space. Then
$$
\left\|A^s B^s\right\| \leq\|A B\|^s, \quad \text { when } 0 \leq s \leq 1 .
$$

\end{proposition}

\begin{proposition}[\cite{features} (Proposition 9)]
\label{ineq3}
Let $\mathcal{H}, \mathcal{K}$ be two separable Hilbert spaces and $X, A$ be bounded linear operators, with $A: \mathcal{H} \rightarrow \mathcal{K}$ and $B: \mathcal{H} \rightarrow \mathcal{H}$ be positive semidefinite.
$$
\left\|A B^\sigma\right\| \leq\|A\|^{1-\sigma}\|A B\|^\sigma, \quad \forall \sigma \in[0,1] .
$$

\end{proposition}

\begin{proposition}
\label{eq4}
Let $H_1, H_2$ be two separable Hilbert spaces and $\mathcal{S}: H_1 \rightarrow H_2$ a compact operator. Then for any function $f:[0,\|\mathcal{S}\|] \rightarrow[0, \infty[$,
$$
f\left(\mathcal{S} \mathcal{S}^*\right) \mathcal{S}=\mathcal{S} f\left(\mathcal{S}^* \mathcal{S}\right).
$$

\end{proposition}
\begin{proof}
The result can be proved using singular value decomposition of a compact operator.
\end{proof}

\begin{proposition}[\cite{spectral.rates} (Lemma 10)]
\label{ineqvolkan}
 Let $L$ be a compact, positive operator on a separable Hilbert space $H$ such that $\|L\| \leq \kappa^2$. Then for any $\lambda \geq 0$,
 \begin{align*}
\left\|(L+\lambda)^\alpha \phi_\lambda(L)\right\| &\leq 2 D \lambda^{-(1-\alpha)}, \quad \forall \alpha \in[0,1],\\
\left\|L^\alpha \phi_\lambda(L)\right\| &\leq  D \lambda^{-(1-\alpha)}, \quad \forall \alpha \in[0,1],
 \end{align*}

 where $D$ is defined in \eqref{def.phi}.
\end{proposition}

\begin{proposition}
\label{ineq5}
With probability at least $1-\delta$, the following bounds hold:
\begin{align*}
\|F^*_\lambda\|_\infty &\le  2\,\kappa^{2r+1}\,R\,D\,\lambda^{-(\frac{1}{2}-r)^+},\\[3pt]
\|F^*_\lambda\|_{\mathcal{H}_M} &\le  2\,\kappa^{2r}\,R\,D\,\lambda^{-(\frac{1}{2}-r)^+},
\end{align*}
provided that 
\[
M \;\ge\; \frac{8\,p\,\kappa^2\,\beta_\infty}{\lambda},
\quad\text{with}\quad
\beta_\infty = \log\!\frac{4\,\kappa^2\bigl(\mathcal{N}_{\mathcal{L}}(\lambda)+1\bigr)}{\delta\,\|\mathcal{L}\|}.
\]
\end{proposition}

\begin{proof}
Since $F_\lambda^* \in \mathcal{H}_M$, we obtain from the reproducing property and the definition $G_\rho = \mathcal{L}^r H$ that, for any $x \in \mathcal{X}$,
\begin{align*}
\|F_\lambda^*(x)\|_\mathcal{Y}
&= \| K_{M,x}^* F_\lambda^* \|_\mathcal{Y}
\;\le\; \kappa\,\|F_\lambda^*\|_{\mathcal{H}_M}
= \kappa\,\|\mathcal{S}^*_M \phi_\lambda(\mathcal{L}_M) G_\rho\|_{\mathcal{H}_M} \\
&= \kappa\,\|\mathcal{L}_M^{\frac{1}{2}} \phi_\lambda(\mathcal{L}_M) \mathcal{L}^r H \|_{L^2(\rho_x)}.
\end{align*}
Using $\|H\|_{L^2(\rho_x)} \le R$, we find
\begin{align}
\|F_\lambda^*\|_\infty
&\le \kappa\,R\,\|\mathcal{L}_M^{\frac{1}{2}}\phi_\lambda(\mathcal{L}_M)\mathcal{L}_{M,\lambda}^{(r\wedge\frac{1}{2})}\|\,
      \|\mathcal{L}_{M,\lambda}^{-(r\wedge\frac{1}{2})}\mathcal{L}^r\|
= \kappa\,R\, (I)\,(II), \label{fbound}\\[4pt]
\|F_\lambda^*\|_{\mathcal{H}_M}
&\le R\,\|\mathcal{L}_M^{\frac{1}{2}}\phi_\lambda(\mathcal{L}_M)\mathcal{L}_{M,\lambda}^{(r\wedge\frac{1}{2})}\|\,
      \|\mathcal{L}_{M,\lambda}^{-(r\wedge\frac{1}{2})}\mathcal{L}^r\|
= R\, (I)\,(II). \label{fbound2}
\end{align}

\paragraph{Step (I).}  
By Proposition~\ref{ineqvolkan},
\begin{align*}
I
&=\bigl\|\mathcal{L}_M^{\frac{1}{2}+(r\wedge\frac{1}{2})}\phi_\lambda(\mathcal{L}_M)\bigr\|
\le
\begin{cases}
D, & r \ge \frac{1}{2},\\[3pt]
D\,\lambda^{r-\frac{1}{2}}, & r < \frac{1}{2},
\end{cases}\\[3pt]
&\le D\,\lambda^{-(\frac{1}{2}-r)^+}.
\end{align*}

\paragraph{Step (II).}  
From Propositions~\ref{ineq2} and~\ref{OPbound5}, with probability at least $1-\delta$, we obtain
\begin{align*}
II
&=
\begin{cases}
\|\mathcal{L}_{M,\lambda}^{-\frac{1}{2}}\mathcal{L}_{\lambda}^{r}\|
   \le \|\mathcal{L}_{M,\lambda}^{-\frac{1}{2}}\mathcal{L}_{\lambda}^{\frac{1}{2}}\|
      \|\mathcal{L}^{r-\frac{1}{2}}\|
   \le 2\,\kappa^{2r-1}, & r \ge \tfrac{1}{2},\\[8pt]
\|\mathcal{L}_{M,\lambda}^{-r}\mathcal{L}_{\lambda}^{r}\|
   \le \|\mathcal{L}_{M,\lambda}^{-\frac{1}{2}}\mathcal{L}_{\lambda}^{\frac{1}{2}}\|^{2r}
   \le 4^r \le 2, & r < \tfrac{1}{2},
\end{cases}\\[3pt]
&\le 2\,\kappa^{2r}.
\end{align*}

Combining the bounds on $I$ and $II$ in~\eqref{fbound} and~\eqref{fbound2} yields
\begin{align*}
\|F_\lambda^*\|_\infty &\le 2\,\kappa^{2r+1}\,R\,D\,\lambda^{-(\frac{1}{2}-r)^+},\\[3pt]
\|F_\lambda^*\|_{\mathcal{H}_M} &\le 2\,\kappa^{2r}\,R\,D\,\lambda^{-(\frac{1}{2}-r)^+}.
\end{align*}
\end{proof}

\begin{proposition}
\label{OPbound1}
Let $\mathcal{H}$ be a separable Hilbert space and let $A$ and $B$  be two bounded self-adjoint positive linear operators on $\mathcal{H}$ and $\lambda>0$. Then

$$
\left\|A_\lambda^{-\frac{1}{2}}B_\lambda ^{\frac{1}{2}}\right\| \leq(1-c)^{-\frac{1}{2}}, \quad \left\|A_\lambda ^{\frac{1}{2}}B_\lambda ^{-\frac{1}{2}}\right\| \leq (1+c)^{\frac{1}{2}},
$$
with
$$
c=\left\|B_\lambda ^{-\frac{1}{2}}(A-B)B_\lambda ^{-\frac{1}{2}}\right\|.
$$
\end{proposition}

\begin{proof}
The proof for the first inequality can for example be found in \cite{features} (Proposition 8). Using simple calculations the second inequality follows from 
\begin{align*}
\left\|(A+\lambda I)^{\frac{1}{2}}(B+\lambda I)^{-\frac{1}{2}}\right\|^2&=\left\|(B+\lambda I)^{-\frac{1}{2}}(A+\lambda I)(B+\lambda I)^{-\frac{1}{2}}\right\|\\
&\leq \left\|(B+\lambda I)^{-\frac{1}{2}}(A-B)(B+\lambda I)^{-\frac{1}{2}}\right\|+\|I\|\leq 1+c.
\end{align*}

\end{proof}

\begin{proposition}[\cite{features} (Lemma 9)]
\label{OPbound3}
For any $M\geq8\kappa^4\|\mathcal{L}\|^{-1}\log^2 \frac{2}{\delta}$  we have with probability at least $1-\delta$
$$
\|\mathcal{L}_M\|\geq\frac{1}{2}\|\mathcal{L}\|.
$$
\end{proposition}
\begin{proof}
For $M\geq8\kappa^4\|\mathcal{L}\|^{-1}\log^2 \frac{2}{\delta}$ we have from Proposition \ref{OPbound2} ($E_6$) that with probability at least $1-\delta$, $\left\|\mathcal{L}-\mathcal{L}_M\right\|_{H S}\leq \frac{1}{2}\|\mathcal{L}\|$ 
and therefore
$$
\|\mathcal{L}_M\|  \geq \|\mathcal{L}\|-\left\|\mathcal{L}-\mathcal{L}_M\right\|_{H S}\geq \frac{1}{2} \|\mathcal{L}\|.
$$
\end{proof}

\begin{proposition}
\label{OPbound5}
Providing Assumption \ref{ass:kernel} 
we have for any $M\geq \frac{8 p\kappa^2 \beta_\infty}{\lambda}$ \\where $\beta_\infty=\log \frac{4 \kappa^2(\mathcal{N}_{\mathcal{L}}(\lambda)+1)}{\delta\|\mathcal{L}\|} $ with probability at least $1-\delta$
$$
\left\|\mathcal{L}_{M,\lambda}^{-\frac{1}{2}}\mathcal{L}_{\lambda}^{\frac{1}{2}}\right\| \leq 2, \quad \left\|\mathcal{L}_{M,\lambda}^{\frac{1}{2}}\mathcal{L}_{\lambda}^{-\frac{1}{2}}\right\| \leq 2.
$$

\end{proposition}

\begin{proof}
From Proposition \ref{OPbound2} ($E_2$) we have for any $\lambda>0$ ,
\begin{align}
\left\|\mathcal{L}_{\lambda}^{-\frac{1}{2}}(\mathcal{L}_M-\mathcal{L})\mathcal{L}_{\lambda}^{-\frac{1}{2}}\right\|\leq \frac{4 \kappa^2 \beta_\infty}{3M \lambda}+\sqrt{\frac{2 p\kappa^2 \beta_\infty}{M\lambda}}.
\end{align}
From $M\geq \frac{8 p\kappa^2 \beta_\infty}{\lambda}$ we therefore obtain
\begin{align}
\left\|\mathcal{L}_{\lambda}^{-\frac{1}{2}}(\mathcal{L}_M-\mathcal{L})\mathcal{L}_{\lambda}^{-\frac{1}{2}}\right\|\leq \frac{3}{4}.
\end{align}

The result now follows from Proposition \ref{OPbound1}
\end{proof}

\begin{proposition}
\label{OPbound6}
Providing Assumption \ref{ass:kernel} we have for  any $n\geq \frac{8\kappa^2 \tilde\beta}{\lambda}$ with\\ $\tilde{\beta}:= \log \frac{4 \kappa^2(\left(1+2\log\frac{2}{\delta}\right)4\mathcal{N}_{\mathcal{L}}(\lambda)+1)}{\delta\|\mathcal{L}\|}$ and $M\geq \frac{8 p\kappa^2 \beta_\infty}{\lambda}\vee 8\kappa^4\|\mathcal{L}\|^{-1}\log^2 \frac{2}{\delta}$, where $\beta_\infty=\log \frac{4 \kappa^2(\mathcal{N}_{\mathcal{L}}(\lambda)+1)}{\delta\|\mathcal{L}\|} $  that with probability at least $1-4\delta$
$$ 
\left\|\widehat{\Sigma}_{M,\lambda}^{-\frac{1}{2}}\Sigma_{M,\lambda}^{\frac{1}{2}}\right\| \leq 2, \quad  \left\|\widehat{\Sigma}_{M,\lambda}^{\frac{1}{2}}\Sigma_{M,\lambda}^{-\frac{1}{2}}\right\| \leq 2 .
$$

\end{proposition}

Note that we use this bound for the variance term in~\eqref{usebound6}. 
Furthermore, a short calculation shows that the above assumption on $n$ is indeed satisfied under the sample size condition stated in Proposition~\ref{mainprop}, namely,
\begin{align*}
n \;\ge\; C_\bullet\,\log^{3(2r+b)}\!\bigl(\delta^{-1}\bigr)\,\lambda^{-(2r+b)},
\qquad 
n \;\ge\; n_0 := \exp\!\left(\tfrac{2r+b}{\,2r+b-1\,}\right).
\end{align*}

\begin{proof}

From Proposition \ref{OPbound2} ($E_1$) we have for any $\lambda>0$  with probability at least $1-\delta$,
\begin{align}
\left\|\Sigma_{M,\lambda}^{-\frac{1}{2}}\left(\widehat{\Sigma}_{M}-\Sigma_{M}\right) \Sigma_{M,\lambda}^{-\frac{1}{2}}\right\|&\leq\frac{4 \kappa^2 \beta_M}{3n \lambda}+\sqrt{\frac{2 \kappa^2 \beta_M}{n\lambda}},\label{e2bb}
\end{align}
with $\beta_M=\log \frac{4 \kappa^2(\mathcal{N}_{\mathcal{L}_M}(\lambda)+1)}{\delta\|\mathcal{L}_M\|}$ . 
For $M\geq \frac{8 p\kappa^2 \beta_\infty}{\lambda}$ we obtain from Proposition \ref{prop:effecdim2} that with probability at least $1-2\delta$,
\begin{align}
\mathcal{N}_{\mathcal{L}_{M}}(\lambda)\leq  \left(1+2\log\frac{2}{\delta}\right)4\mathcal{N}_{\mathcal{L}}(\lambda). \label{boundE24}
\end{align}
From Proposition \ref{OPbound3} we have with probability $1-\delta,$

\begin{align}
\|\mathcal{L}_M\|\geq\frac{1}{2}\|\mathcal{L}\|. \label{boundE56}
\end{align}

Note that the bounds of \eqref{boundE24} and \eqref{boundE56} imply $\beta_M\leq\tilde{\beta}= \log \frac{4 \kappa^2(\left(1+2\log\frac{2}{\delta}\right)4\mathcal{N}_{\mathcal{L}}(\lambda)+1)}{\delta\|\mathcal{L}\|}$ . Using this together with $n\geq \frac{8\kappa^2 \tilde\beta}{\lambda}$ we obtain for \eqref{e2bb}

\begin{align}
\left\|\Sigma_{M,\lambda}^{-\frac{1}{2}}\left(\widehat{\Sigma}_{M}-\Sigma_{M}\right) \Sigma_{M,\lambda}^{-\frac{1}{2}}\right\|&\leq\frac{4 \kappa^2 \beta_M}{3n \lambda}+\sqrt{\frac{2 \kappa^2 \beta_M}{n\lambda}}\\
&\leq\frac{4 \kappa^2 \tilde\beta}{3n \lambda}+\sqrt{\frac{2 \kappa^2 \tilde\beta}{n\lambda}}\leq \frac{3}{4}.
\end{align}
Note that from Proposition \ref{conditioning} we have that the above inequality holds with probability at least $1-4\delta$.
The result now follows from Proposition \ref{OPbound1}
\end{proof}

\begin{proposition}
\label{OPbound7}
Providing Assumption \ref{ass:kernel} we have for any  
\begin{align*}
M\geq 
\begin{cases}
\frac{8 p\kappa^2 \beta_\infty}{\lambda} & r\in\left(0,\frac{1}{2}\right)\\
\frac{(8 p\kappa^2 \beta_\infty)\vee C_1^{\frac{1}{r}}}{\lambda}\vee \frac{C_2}{\lambda^{1+b(2r-1)} } & r\in\left[\frac{1}{2},1\right] \\
\frac{C_3}{\lambda^{2r}} & r \in(1,\infty)
\end{cases}
\end{align*}
with probability at least $1-3\delta$,
$$
\left\|\mathcal{L}_{M,\lambda}^{-(r\vee1)}\mathcal{L}_{\lambda}^{r}\right\| \leq \frac{3}{\lambda^{(1-r)^+}},
$$
where  $C_1=2(4\kappa\log\frac{2}{\delta})^{2r-1}(8p\kappa^2\beta_\infty)^{1-r}$ ,  $C_2=4(4c_b\kappa^2\log\frac{2}{\delta})^{2r-1}(8p\kappa^2\beta_\infty)^{2-2r}$, \\ $C_3:= 4\kappa^4C_{\kappa,r}^2\log^2\frac{2}{\delta}$ and with $C_{\kappa,r}$ from Proposition \ref{ineq1}.
\end{proposition}

\begin{proof} For the proof we need to differ between the following three cases:
 \begin{itemize}

 \item CASE ($r\leq \frac{1}{2}$) :  From Proposition \ref{OPbound5} we have with probability at least $1-\delta$,
\begin{align*}
\left\|\mathcal{L}_{M,\lambda}^{-(r\vee1)}\mathcal{L}_{\lambda}^{r}\right\|&=\left\|\mathcal{L}_{M,\lambda}^{-1}\mathcal{L}_{\lambda}^{r}\right\| \\
&\leq \lambda^{r-1}\left\|\mathcal{L}_{M,\lambda}^{-r}\mathcal{L}_{\lambda}^{r}\right\|\\
&\leq \lambda^{r-1}\left\|\mathcal{L}_{M,\lambda}^{-\frac{1}{2}}\mathcal{L}_{\lambda}^{\frac{1}{2}}\right\|^{2r} \leq 2^{2r}\lambda^{r-1}\leq 3\lambda^{r-1}.
\end{align*}

\item CASE ($r\in[\frac{1}{2},1]$) :  Using $\left\|\mathcal{L}_{\lambda}^{-1}\mathcal{L}_{\lambda}^{r}\right\|\leq \lambda^{r-1}$ we have

\begin{align}
\left\|\mathcal{L}_{M,\lambda}^{-(r\vee1)}\mathcal{L}_{\lambda}^{r}\right\|&=\left\|\mathcal{L}_{M,\lambda}^{-1}\mathcal{L}_{\lambda}^{r}\right\| \\
&\leq\left\|\left(\mathcal{L}_{M,\lambda}^{-1}-\mathcal{L}_{\lambda}^{-1}\right)\mathcal{L}_{\lambda}^{r}\right\|+\lambda^{r-1}.\label{OP8i}
\end{align}

For the norm of the last inequality we have from the algebraic identity 
\\$A^{-1}-B^{-1}=A^{-1}(A-B)B^{-1}$:
\begin{align*}
\left\|\left(\mathcal{L}_{M,\lambda}^{-1}-\mathcal{L}_{\lambda}^{-1}\right)\mathcal{L}_{\lambda}^{r}\right\|
=\left\|\mathcal{L}_{M,\lambda}^{-1}\left(\mathcal{L}_{M,\lambda}-\mathcal{L}_{\lambda}\right)\mathcal{L}_{\lambda}^{r-1}\right\|
\end{align*}
and from Proposition \ref{OPbound5} we further have with probability at least $1-\delta$,
\begin{align*}
&\left\|\mathcal{L}_{M,\lambda}^{-1}\left(\mathcal{L}_{M,\lambda}-\mathcal{L}_{\lambda}\right)\mathcal{L}_{\lambda}^{r-1}\right\|\\
&\leq \lambda^{-\frac{1}{2}}\left\|\mathcal{L}_{M,\lambda}^{-\frac{1}{2}}\mathcal{L}_{\lambda}^{\frac{1}{2}}\right\| \left\|\mathcal{L}_{\lambda}^{-\frac{1}{2}}\left(\mathcal{L}_{M,\lambda}-\mathcal{L}_{\lambda}\right)\mathcal{L}_{\lambda}^{r-1}\right\|\\
&\leq 2\lambda^{-\frac{1}{2}} \left\|\mathcal{L}_{\lambda}^{-\frac{1}{2}}\left(\mathcal{L}_{M,\lambda}-\mathcal{L}_{\lambda}\right)\mathcal{L}_{\lambda}^{r-1}\right\|.\\
\end{align*}
Since $\sigma:=2-2r\leq 1$ we have from Proposition \ref{ineq3}  
\begin{align*}
& \left\|\mathcal{L}_{\lambda}^{-\frac{1}{2}}\left(\mathcal{L}_{M,\lambda}-\mathcal{L}_{\lambda}\right)\mathcal{L}_{\lambda}^{r-1}\right\|\\
 &\leq \left\|\mathcal{L}_{\lambda}^{-\frac{1}{2}}\left(\mathcal{L}_{M,\lambda}-\mathcal{L}_{\lambda}\right)\right\|^{2r-1}  \left\|\mathcal{L}_{\lambda}^{-\frac{1}{2}}\left(\mathcal{L}_{M,\lambda}-\mathcal{L}_{\lambda}\right)\mathcal{L}_{\lambda}^{-\frac{1}{2}}\right\|^{2-2r}.
\end{align*}
Using Proposition \ref{OPbound2} ($E_2$ and $E_5$) we have for the last expression with probability at least $1-2\delta$,
\begin{align*}
\leq \left[\left(\frac{2 
\kappa}{\sqrt{\lambda}M}+\sqrt{\frac{4 \kappa^2 \mathcal{N}_{\mathcal{L}}(\lambda) }{M}}\right)\log \frac{2}{\delta}\right]^{2r-1}
\left(\frac{4 \kappa^2 \beta_\infty}{3M \lambda}+\sqrt{\frac{2 p\kappa^2 \beta_\infty}{M\lambda}}\right)^{2-2r},
\end{align*}
with $\beta_\infty=\log \frac{4 \kappa^2(\mathcal{N}_{\mathcal{L}}(\lambda)+1)}{\delta\|\mathcal{L}\|}$. Using this together with  $M\geq \frac{8 p\kappa^2 \beta_\infty}{\lambda} $ and the simple inequality $(a+b)^{2r-1}\leq a^{2r-1} + b^{2r-1}$ we have
\begin{align*}
&\left\|\left(\mathcal{L}_{M,\lambda}^{-1}-\mathcal{L}_{\lambda}^{-1}\right)\mathcal{L}_{\lambda}^{r}\right\|\\
&\leq 2\lambda^{-\frac{1}{2}}\left(\frac{4 
\kappa \log \frac{2}{\delta}}{\sqrt{\lambda}M}+\sqrt{\frac{4 \kappa^2 \mathcal{N}_{\mathcal{L}}(\lambda) \log \frac{2}{\delta}}{M}}\right)^{2r-1}
\left(\frac{4 \kappa^2 \beta_\infty}{3M \lambda}+\sqrt{\frac{2 p\kappa^2 \beta_\infty}{M\lambda}}\right)^{2-2r}\\
& \leq2\lambda^{-\frac{1}{2}} \left(\frac{4 
\kappa \log \frac{2}{\delta}}{\sqrt{\lambda}M}+\sqrt{\frac{4 \kappa^2 \mathcal{N}_{\mathcal{L}}(\lambda) \log \frac{2}{\delta}}{M}}\right)^{2r-1}
\left(2\sqrt{\frac{2p\kappa^2 \beta_\infty}{M\lambda}}\right)^{2-2r}\\
&\leq\frac{C_1}{\lambda M^r}+\sqrt{\frac{C_2'\mathcal{N}_{\mathcal{L}}(\lambda)^{2r-1}}{M\lambda^{3-2r}}}\leq\frac{C_1}{\lambda M^r}+\sqrt{\frac{C_2}{M\lambda^{3-2r+b(2r-1)}}},
\end{align*}
where we used in the last inequality the assumption $\mathcal{N}_{\mathcal{L}}(\lambda)\leq c_b \lambda^{-b}$  and set \\$C_1=2(4\kappa\log\frac{2}{\delta})^{2r-1}(8p\kappa^2\beta_\infty)^{1-r}$ ,  $C_2=4(4c_b\kappa^2\log^2\frac{2}{\delta})^{2r-1}(8p\kappa^2\beta_\infty)^{2-2r}$.
From $M\geq \frac{C_1^{\frac{1}{r}}}{\lambda} $ and $M\geq \frac{C_2}{\lambda^{1+b(2r-1)} }$ we obtain 
\begin{align*}
\left\|\left(\mathcal{L}_{M,\lambda}^{-1}-\mathcal{L}_{\lambda}^{-1}\right)\mathcal{L}_{\lambda}^{r}\right\|
\leq\frac{C_1}{\lambda M^r}+\sqrt{\frac{C_2}{M\lambda^{3-2r+b(2r-1)}}}
\leq 2\lambda^{r-1}.
\end{align*}
Plugging this bound into \eqref{OP8i} leads to
\begin{align*}
\left\|\mathcal{L}_{M,\lambda}^{-(r\vee1)}\mathcal{L}_{\lambda}^{r}\right\|\leq3 \lambda^{r-1}.
\end{align*}

\item CASE $(r\geq 1)$ :  
\begin{align*}
\left\|\mathcal{L}_{M,\lambda}^{-(r\vee1)}\mathcal{L}_{\lambda}^{r}\right\|&=\left\|\mathcal{L}_{M,\lambda}^{-r}\mathcal{L}_{\lambda}^{r}\right\|\\
&\leq 1+ \left\|\mathcal{L}_{M,\lambda}^{-r}\left(\mathcal{L}_{\lambda}^{r}-\mathcal{L}_{M,\lambda}^{r}\right)\right\|\\
&\leq 1+ \lambda^{-r}C_{\kappa, r}\left\|\mathcal{L}_{\lambda}-\mathcal{L}_{M,\lambda}\right\|,
\end{align*}
where $C_{\kappa, r}$ is defined in Proposition \ref{ineq1}.
From the bound of Proposition \ref{OPbound2} ($E_6$) we therefore obtain
\begin{align*}
&\left\|\mathcal{L}_{M,\lambda}^{-(r\vee1)}\mathcal{L}_{\lambda}^{r}\right\|\\
&\leq 1+ \lambda^{-r}C_{1, r}  \left(\frac{2 \kappa^2}{M} + \frac{2 \kappa^2}{\sqrt{M}} \right)\log \frac{2}{\delta}\leq 3,
\end{align*}
where used $M\geq C_3 \lambda^{-2r}$, with $C_3:= 4\kappa^4C_{1,r}^2\log^2\frac{2}{\delta}$.
\end{itemize} 
\end{proof}

\begin{proposition}
\label{OPbound8}
  For any $q>0$, $n\geq \max\{8C_{\kappa,q}^2 \kappa^4\lambda^{-2q} \log^2\frac{2}{\delta},\, 100\kappa^2 \mathcal{N}_{\mathcal{L}}(\lambda)\lambda^{-1} \log^3 \frac{2}{\delta}\}$ and $M\geq \frac{8 p\kappa^2 \beta_\infty}{\lambda}$ we have with probability at least $1-\delta$,

\begin{align*}
\left\|\widehat{\Sigma}_{M,\lambda}^{-q} \Sigma_{M,\lambda}^{q}\right\|\leq 2.
\end{align*}
\end{proposition}

Note that we use this bound for the variance term in~\eqref{usebound8}, with $q = (r \vee 1) - \tfrac{1}{2}$. 
Furthermore, the above assumption on $n$ is automatically satisfied in the case $q = (r \vee 1) - \tfrac{1}{2}$ (with $r \ge \tfrac{1}{2}$) 
by the sample size condition stated in Proposition~\ref{mainprop}, namely
\[
n \;\ge\; C_\bullet\,\log^{3(2r+b)}\!\bigl(\delta^{-1}\bigr)\,\lambda^{-(2r+b)}.
\]

\begin{proof}
\begin{itemize}
\item \textbf{Case $q<1$:}

From Proposition \ref{OPbound2}($E_3$) we obtain with probability at least $1-\delta$,
\begin{align*}
\left\|\widehat{\Sigma}_{M,\lambda}^{-q} \Sigma_{M,\lambda}^{q}\right\|&=\left\|\widehat{\Sigma}_{M,\lambda}^{-1} \Sigma_{M,\lambda}\right\|^q\\
&\leq\left\|\widehat{\Sigma}_{M,\lambda}^{-1}\left(\widehat{\Sigma}_{M}-\Sigma_{M}\right)\right\|_{HS}+1\\
&\leq\frac{1}{\sqrt{\lambda}} \left\|\widehat{\Sigma}_{M,\lambda}^{-\frac{1}{2}} \Sigma_{M,\lambda}^{\frac{1}{2}}\right\|\left\|\Sigma_{M,\lambda}^{-\frac{1}{2}}\left(\widehat{\Sigma}_{M}-\Sigma_{M}\right)\right\|_{HS}+1\\
&\leq \frac{1}{\sqrt{\lambda}} \left\|\widehat{\Sigma}_{M,\lambda}^{-\frac{1}{2}} \Sigma_{M,\lambda}^{\frac{1}{2}}\right\| \left(\frac{2\kappa}{\sqrt{\lambda} n}+\sqrt{\frac{ 4\kappa^2 \mathcal{N}_{\mathcal{L}_M}(\lambda) }{ n}}\right)\log \frac{2}{\delta}+1.
\end{align*}
We have from Proposition \ref{OPbound6} with probability at least $1-\delta$,
$$ 
\left\|\widehat{\Sigma}_{M,\lambda}^{-\frac{1}{2}}\Sigma_{M,\lambda}^{\frac{1}{2}}\right\| \leq 2
$$
and therefore
\begin{align}
\left\|\widehat{\Sigma}_{M,\lambda}^{-q} \Sigma_{M,\lambda}^{q}\right\|\leq \frac{2}{\sqrt{\lambda}} \left(\frac{2\kappa}{\sqrt{\lambda} n}+\sqrt{\frac{ 4\kappa^2 \mathcal{N}_{\mathcal{L}_M}(\lambda) }{ n}}\right)\log \frac{2}{\delta}+1.\label{ineqT2ii}
\end{align}

From  \ref{prop:effecdim2} we have with probability at least $1-2\delta$,
$$
\mathcal{N}_{\mathcal{L}_{M}}(\lambda)\leq  \left(1+2\log\frac{2}{\delta}\right)4\mathcal{N}_{\mathcal{L}}(\lambda).
$$
Plugging this bound into \eqref{ineqT2ii} leads to
\begin{align}
\left\|\widehat{\Sigma}_{M,\lambda}^{-q} \Sigma_{M,\lambda}^{q}\right\|\leq \frac{2}{\sqrt{\lambda}} \left(\frac{2\kappa}{\sqrt{\lambda} n}+\sqrt{\frac{ 4\kappa^2 \left(1+2\log\frac{2}{\delta}\right)4\mathcal{N}_{\mathcal{L}}(\lambda) }{ n}}\right)\log \frac{2}{\delta}+1\leq 2,
\end{align}

where we used $n\geq 100\kappa^2 \mathcal{N}_{\mathcal{L}}(\lambda)\lambda^{-1} \log^3 \frac{2}{\delta}$ in the last inequality.

\item \textbf{Case $q\geq1$:}   From Proposition \ref{ineq1} and Proposition \ref{OPbound2}($E_7$) we have with probability at least $1-\delta$,

\begin{align*}
\left\|\widehat{\Sigma}_{M,\lambda}^{-q} \Sigma_{M,\lambda}^q\right\|
&\leq\lambda^{-q}\left\|\widehat{\Sigma}_{M}^q-\Sigma_{M}^q\right\|_{HS}+1\\
&\leq  \lambda^{-q}C_{\kappa,q}\left\|\widehat{\Sigma}_{M}-\Sigma_{M}\right\|_{HS}+1\\
&\leq  \lambda^{-q}C_{\kappa,q}\left(\frac{2 \kappa^2}{n} + \frac{2 \kappa^2}{\sqrt{n}} \right)\log \frac{2}{\delta}+1\leq2,
\end{align*}
where we used $n\geq 8C_{\kappa,r}^2 \kappa^4\lambda^{-2q} \log^2\frac{2}{\delta}$ for the last inequality.
\end{itemize}
\end{proof}


\begin{proposition}
\label{prop:effecdim2}
For any $M\geq \frac{8 p\kappa^2 \beta_\infty}{\lambda}$ we have with probability at least $1-2\delta$,
$$
\mathcal{N}_{\mathcal{L}_{M}}(\lambda)\leq  \left(1+2\log\frac{2}{\delta}\right)4\mathcal{N}_{\mathcal{L}}(\lambda).
$$

\end{proposition}

\begin{proof}
\begin{align*}
\mathcal{N}_{\mathcal{L}_{M}}(\lambda)&\leq \text{ Tr}[\mathcal{L}_M\mathcal{L}_{\lambda}^{-1}]\left\|\mathcal{L}_{\lambda}^{\frac{1}{2}}\mathcal{L}_{M,\lambda}^{-\frac{1}{2}}\right\|^2\\
&=\left(\mathcal{N}_{\mathcal{L}}+\text{ Tr}[(\mathcal{L}_M-\mathcal{L})\mathcal{L}_{\lambda}^{-1}]\right)\left\|\mathcal{L}_{\lambda}^{\frac{1}{2}}\mathcal{L}_{M,\lambda}^{-\frac{1}{2}}\right\|^2\\
&=\left(\mathcal{N}_{\mathcal{L}}+\|B\|_{HS}\right)\left\|\mathcal{L}_{\lambda}^{\frac{1}{2}}\mathcal{L}_{M,\lambda}^{-\frac{1}{2}}\right\|^2,
\end{align*}

where $B:=\mathcal{L}_{\lambda}^{-\frac{1}{2}}(\mathcal{L}_M-\mathcal{L})\mathcal{L}_{\lambda}^{-\frac{1}{2}}$.  Proposition \ref{OPbound2}($E_4$) we have with probability at least $1-\delta$,
$$
\|B\|_{HS}\leq   2\left(\frac{2\kappa^2}{\lambda M}+\sqrt{\frac{ \kappa^2 \mathcal{N}_{\mathcal{L}}(\lambda) }{\lambda M}}\right)\log \frac{2}{\delta}.
$$
Using $\lambda> 4\kappa^2M^{-1}$ we obtain
$$
\|B\|_{HS}\leq 2\mathcal{N}_{\mathcal{L}}(\lambda) \log \frac{2}{\delta}.
$$
Further we have from Proposition \ref{OPbound5} with probability at least $1-\delta$,

$$
\left\|\mathcal{L}_{\lambda}^{\frac{1}{2}}\mathcal{L}_{M,\lambda}^{-\frac{1}{2}}\right\|^2\leq4.
$$

To sum up, we obtain 
\begin{align*}
\mathcal{N}_{\mathcal{L}_{M}}(\lambda)\leq\left(\mathcal{N}_{\mathcal{L}}+\|B\|_{HS}\right)\left\|\mathcal{L}_{\lambda}^{\frac{1}{2}}\mathcal{L}_{M,\lambda}^{-\frac{1}{2}}\right\|^2\leq \left(1+2\log\frac{2}{\delta}\right)4\mathcal{N}_{\mathcal{L}}(\lambda).
\end{align*}

\end{proof}


\input{K.U}


%% file: K.U.tex

\subsection{Concentration Inequalities}
\label{concineq}

\begin{proposition}
\label{OPbound0}
Let $\mathcal{X}_1, \cdots, \mathcal{X}_m$ be a sequence of independently and identically distributed selfadjoint Hilbert-Schmidt operators on a separable Hilbert space. Assume that $\mathbb{E}\left[\mathcal{X}_1\right]=0$, and $\left\|\mathcal{X}_1\right\| \leq B$ almost surely for some $B>0$. Let $\mathcal{V}$ be a positive trace-class operator such that $\mathbb{E}\left[\mathcal{X}_1^2\right] \preccurlyeq \mathcal{V}$. Then with probability at least $1-\delta,(\delta \in] 0,1[)$, there holds
$$
\left\|\frac{1}{m} \sum_{i=1}^m \mathcal{X}_i\right\| \leq \frac{2 B \beta}{3 m}+\sqrt{\frac{2\|\mathcal{V}\| \beta}{m}}, \quad \beta=\log \frac{4 \operatorname{tr} \mathcal{V}}{\|\mathcal{V}\| \delta}.
$$
\end{proposition}
\begin{proof}
The proposition was first established for matrices by  \cite{Tropp_2011}. For the general case including operators the proof can for example be found in  \cite{spectral.rates} (see Lemma 26).
\end{proof}

\begin{proposition}
\label{concentrationineq0}
The following concentration result for Hilbert space valued random variables can be found in (Caponnetto and De Vito, 2007 \cite{Caponetto}).\\
\\
Let $w_{1}, \cdots, w_{n}$ be i.i.d random variables in a separable Hilbert space with norm $\|.\|$. Suppose that there are two positive constants $B$ and $\sigma^{2}$ such that
\begin{align}
\mathbb{E}\left[\left\|w_{1}-\mathbb{E}\left[w_{1}\right]\right\|^{l}\right] \leq \frac{1}{2} l ! B^{l-2} V^{2}, \quad \forall l \geq 2 .\label{cons}
\end{align}
Then for any $0<\delta<1 / 2$, the following holds with probability at least $1-\delta$,
$$
\left\|\frac{1}{n} \sum_{k=1}^{n} w_{n}-\mathbb{E}\left[w_{1}\right]\right\| \leq \left(\frac{2B}{n}+\frac{2V}{\sqrt{n}}\right) \log \frac{2}{\delta} .
$$
In particular, (\ref{cons}) holds if
$$
\left\|w_{1}\right\| \leq B / 2 \quad \text { a.s., } \quad \text { and } \quad \mathbb{E}\left[\left\|w_{1}\right\|^{2}\right] \leq V^{2} .
$$
\end{proposition}

\begin{proposition}
\label{OPbound2}
For any $\lambda>0$ define the following events,\vspace{0.4cm}\\

$\begin{aligned}
&E_1=\left\{\left\|\Sigma_{M,\lambda}^{-\frac{1}{2}}\left(\widehat{\Sigma}_{M}-\Sigma_{M}\right) \Sigma_{M,\lambda}^{-\frac{1}{2}}\right\|\leq\frac{4 \kappa^2 \beta_M}{3n \lambda}+\sqrt{\frac{2 \kappa^2 \beta_M}{n\lambda}}\right\},  &&\beta_M=\log \frac{4 \kappa^2(\mathcal{N}_{\mathcal{L}_M}(\lambda)+1)}{\delta\|\mathcal{L}_M\|},\\[7pt]
&E_2=\left\{\left\|\mathcal{L}_{\lambda}^{-\frac{1}{2}}(\mathcal{L}_M-\mathcal{L})\mathcal{L}_ {\lambda}^{-\frac{1}{2}}\right\|\leq \frac{4 \kappa^2 \beta_\infty}{3M \lambda}+\sqrt{\frac{2 p\kappa^2 \beta_\infty}{M\lambda}}\right\}, &&\beta_\infty=\log \frac{4 \kappa^2(\mathcal{N}_{\mathcal{L}}(\lambda)+1)}{\delta\|\mathcal{L}\|}  ,\\
&E_3=\left\{\left\|\Sigma_{M,\lambda}^{-\frac{1}{2}}\left(\widehat{\Sigma}_{M}-\Sigma_{M}\right)\right\|_{HS}\leq  \left(\frac{2\kappa}{\sqrt{\lambda} n}+\sqrt{\frac{ 4\kappa^2 \mathcal{N}_{\mathcal{L}_M}(\lambda) }{ n}}\right)\log \frac{2}{\delta}\right\},\\
&E_4=\left\{\left\|\mathcal{L}_ {\lambda}^{-\frac{1}{2}}(\mathcal{L}_M-\mathcal{L})\mathcal{L}_ {\lambda}^{-\frac{1}{2}}\right\|_{HS}\leq  \left(\frac{4\kappa^2}{\lambda M}+\sqrt{\frac{ 4\kappa^2 \mathcal{N}_{\mathcal{L}}(\lambda) }{\lambda M}}\right)\log \frac{2}{\delta}\right\},\\
\end{aligned}$\\
$\begin{aligned}
&E_5=\left\{\left\|\mathcal{L}_ {\lambda}^{-\frac{1}{2}}\left(\mathcal{L}_M-\mathcal{L}\right)\right\| \leq \left(\frac{2 
\kappa}{\sqrt{\lambda}M}+\sqrt{\frac{4 \kappa^2 \mathcal{N}_{\mathcal{L}}(\lambda) }{M}}\right)\log \frac{2}{\delta}\right\},\\
&E_6=\left\{\left\|\mathcal{L}-\mathcal{L}_M\right\|_{H S} \leq \left(\frac{2 \kappa^2}{M} + \frac{2 \kappa^2}{\sqrt{M}} \right)\log \frac{2}{\delta}\right\}\,,\\
&E_7=\left\{\left\|\widehat{\Sigma}_{M}-\Sigma_{M}\right\|_{HS} \leq \left(\frac{2 \kappa^2}{n} + \frac{2 \kappa^2}{\sqrt{n}} \right)\log \frac{2}{\delta}\right\}\,.
\end{aligned}$\vspace{0.4cm}\\
Providing Assumption \ref{ass:input}  we have for any $\delta \in(0,1)$ that each of the above events holds true with probability at least $1-\delta$  .
\end{proposition}
\begin{proof}
The bound for $E_1$  follows exactly the same steps as in the proof of \cite{spectral.rates} (Lemma 18). The events $E_2-E_7$ have been bounded in \cite{features} ( see Proposition 6, Lemma 8 and Proposition 10). However, due to different assumptions and a different setting we attain slightly different bounds and therefore give the proof of  the events $E_2-E_7$ for completeness.\\

\textbf{$E_2)$} First note that $\mathcal{L}_M$ can be expressed as
\[
\mathcal{L}_M
= \frac{1}{M}\sum_{m=1}^M \sum_{i=1}^p 
   \varphi_m^{(i)} \otimes \varphi_m^{(i)},
\]
where $\varphi_m(\cdot) = \varphi(\cdot, \omega_m)$, and the tensor product is taken with respect to the $L^2(\mathcal{U},\rho_\mathcal{U})$ inner product.  
The above identity follows by straightforward calculation:
\begin{align*}
(\mathcal{L}_M G)(u)
&= \int K_M(u,\tilde{u})\,G(\tilde{u})\,d\rho_\mathcal{U}(\tilde{u}) \\[2pt]
&= \frac{1}{M}\sum_{m=1}^M \sum_{i=1}^p 
   \int \varphi_m^{(i)}(u)\otimes_\mathcal{V} 
        \varphi_m^{(i)}(\tilde{u})\,G(\tilde{u})\,d\rho_\mathcal{U}(\tilde{u}) \\[2pt]
&= \frac{1}{M}\sum_{m=1}^M \sum_{i=1}^p 
   \varphi_m^{(i)}(u)
   \int \!\bigl\langle \varphi_m^{(i)}(\tilde{u}), G(\tilde{u}) \bigr\rangle_\mathcal{V}
   d\rho_\mathcal{U}(\tilde{u}) \\[2pt]
&= \frac{1}{M}\sum_{m=1}^M \sum_{i=1}^p 
   \bigl(\varphi_m^{(i)} \otimes \varphi_m^{(i)}\bigr)(G)(u).
\end{align*}

Similarly, we have $\mathcal{L}=\mathbb{E}[ \sum_{i=1}^p \varphi_i\otimes\varphi_i]$.

Now define $\mathcal{X}_m:= \mathcal{L}_ {\lambda}^{-\frac{1}{2}}(\mathcal{L}_{M}^{(m)} -\mathcal{L})\mathcal{L}_ {\lambda}^{-\frac{1}{2}}$, with $\mathcal{L}_{M}^{(m)}:=\sum_{i=1}^p\varphi_m^{(i)}\otimes\varphi_m^{(i)}$.
We now obtain
\begin{align*}
\|\mathcal{X}_1\|\leq \left\|\mathcal{L}_ {\lambda}^{-\frac{1}{2}}\mathcal{L}_{M}^{(m)} \mathcal{L}_ {\lambda}^{-\frac{1}{2}}\right\| + \mathbb{E}\left\|\mathcal{L}_ {\lambda}^{-\frac{1}{2}}\mathcal{L}_{M}^{(m)} \mathcal{L}_ {\lambda}^{-\frac{1}{2}}\right\|\leq 2\frac{\kappa^2}{\lambda}:=B,
\end{align*}
where we used for the last inequality
$$
 \left\|\mathcal{L}_ {\lambda}^{-\frac{1}{2}}\mathcal{L}_{M}^{(m)} \mathcal{L}_ {\lambda}^{-\frac{1}{2}}\right\|\leq \lambda^{-1} \left\|\mathcal{L}_{M}^{(m)} \right\|\leq \frac{\kappa^2}{\lambda}.
$$
For the second moment we have from Jensen-inequality

\begin{align*}
\mathbb{E}\left[\mathcal{X}^2\right] &\preccurlyeq \mathbb{E}\left[\left(\mathcal{L}_ {\lambda}^{-\frac{1}{2}}\mathcal{L}_{M}^{(m)} \mathcal{L}_ {\lambda}^{-\frac{1}{2}}\right)^2\right]\\
&\preccurlyeq  \mathbb{E}\left[p\sum_{i=1}^p\left(\mathcal{L}_ {\lambda}^{-\frac{1}{2}}\varphi_m^{(i)}\otimes\varphi_m^{(i)} \mathcal{L}_ {\lambda}^{-\frac{1}{2}}\right)^2\right]\\
&=  \mathbb{E}\left[p\sum_{i=1}^p\left\|\mathcal{L}_ {\lambda}^{-\frac{1}{2}}\varphi_m^{(i)}\right\|^2_{L^2_{\rho_x}}\mathcal{L}_ {\lambda}^{-\frac{1}{2}}\varphi_m^{(i)}\otimes\varphi_m^{(i)} \mathcal{L}_ {\lambda}^{-\frac{1}{2}}\right]\\
&\preccurlyeq \mathbb{E}\left[p\frac{\kappa^2}{\lambda}\mathcal{L}_ {\lambda}^{-\frac{1}{2}}\mathcal{L}^{(m)}_{M}\mathcal{L}_ {\lambda}^{-\frac{1}{2}}\right]\\
&= \frac{p\kappa^2}{\lambda}\mathcal{L}\mathcal{L}_ {\lambda}^{-1}:=\mathcal{V}
\end{align*}

For $\beta=\log \frac{4 \operatorname{tr} \mathcal{V}}{\|\mathcal{V}\| \delta}$ we have
\begin{align*}
\beta&=\log \frac{4 \mathcal{N}_{\mathcal{L}}(\lambda)}{\|\mathcal{L}\mathcal{L}_ {\lambda}^{-1}\| \delta}\\
&=\log \frac{4 \mathcal{N}_{\mathcal{L}}(\lambda)(\|\mathcal{L}\|+\lambda)}{\|\mathcal{L}\| \delta}\\[7pt]
&\leq \log \frac{4 \mathcal{N}_{\mathcal{L}}(\lambda)\|\mathcal{L}\|+4\operatorname{tr}\mathcal{L}}{\|\mathcal{L}\| \delta}\leq \log \frac{4\kappa^2 (\mathcal{N}_{\mathcal{L}}(\lambda)+1)}{\|\mathcal{L}\|\delta}.
\end{align*}

The claim now follows from Proposition \ref{OPbound0}.\\

\textbf{$E_3)$} Set $w_i := \Sigma_{M,\lambda}^{-\frac{1}{2}}\xi_i$, where $\xi_i := K_{M,u_i}K_{M,u_i}^*$. 
Note that $\mathbb{E}[\xi_i] = \Sigma_M$. Then, we have
\begin{align*}
\|w_i\|_{\mathrm{HS}}
&= \bigl\|\Sigma_{M,\lambda}^{-\frac{1}{2}} K_{M,u_i}K_{M,u_i}^*\bigr\|_{\mathrm{HS}}\\[2pt]
&\le \|\Sigma_{M,\lambda}^{-\frac{1}{2}}\|\, \|K_{M,u_i}K_{M,u_i}^*\|_{\mathrm{HS}} \\[2pt]
&\le  \lambda^{-\frac{1}{2}}\, \|K_M(u_i,u_i)\|_{\mathrm{HS}}
 \;\le\; \frac{\kappa^2}{\sqrt{\lambda}} \;=: B.
\end{align*}

For the second moment, we obtain
\begin{align*}
\mathbb{E}\|w_i\|_{\mathrm{HS}}^2
&= \mathbb{E}\,\operatorname{tr}\!\left[
  \Sigma_{M,\lambda}^{-\frac{1}{2}} \xi_i 
  \Sigma_{M,\lambda}^{-1} \xi_i 
  \Sigma_{M,\lambda}^{-\frac{1}{2}}
\right]\\[3pt]
&\le \kappa^2 \,\mathbb{E}\,\operatorname{tr}\!\left[
  \Sigma_{M,\lambda}^{-\frac{1}{2}} K_{M,u_i}K_{M,u_i}^* 
  \Sigma_{M,\lambda}^{-\frac{1}{2}}
\right]\\[2pt]
&= \kappa^2\, \operatorname{tr}\!\bigl(\Sigma_{M,\lambda}^{-1}\Sigma_M\bigr)
 = \kappa^2\, \mathcal{N}_{\mathcal{L}_M}(\lambda)
 \;=: V^2.
\end{align*}

The claim then follows directly by applying Proposition~\ref{concentrationineq0}.

\textbf{$E_4)$}  Set $w_m:= \mathcal{L}_ {\lambda}^{-\frac{1}{2}}(\mathcal{L}_{M}^{(m)} -\mathcal{L})\mathcal{L}_ {\lambda}^{-\frac{1}{2}}$ . Note that we have
\begin{align*}
\|w_m\|_{HS}&\leq \left\|\mathcal{L}_ {\lambda}^{-\frac{1}{2}}\mathcal{L}_{M}^{(m)} \mathcal{L}_ {\lambda}^{-\frac{1}{2}}\right\|_{HS} + \operatorname{tr}\left[\mathcal{L} \mathcal{L}_ {\lambda}^{-1}\right]\\
&\leq \left\|\mathcal{L}_ {\lambda}^{-\frac{1}{2}} \left(\sum_{i=1}^p\varphi_m^{(i)}\otimes\varphi_m^{(i)}\right)\mathcal{L}_ {\lambda}^{-\frac{1}{2}}\right\|_{HS} + \mathcal{N}_{\mathcal{L}}(\lambda)\\
&\leq \lambda^{-1}\sum_{i=1}^p\left\| \varphi_m^{(i)}\otimes\varphi_m^{(i)} \right\|_{HS} + \mathcal{N}_{\mathcal{L}}(\lambda)\\
&\leq \lambda^{-1}\sum_{i=1}^p\left\| \varphi_m^{(i)} \right\|_{L^2_{\rho_x}}^2 + \mathcal{N}_{\mathcal{L}}(\lambda)\leq\frac{2\kappa^2}{\lambda}=:B.
\end{align*}

For the second moment we have,

\begin{align*}
\mathbb{E}\left\|w_m\right\|_{HS}^2\leq\mathbb{E}\operatorname{tr}\left[\left(\mathcal{L}_ {\lambda}^{-\frac{1}{2}}\mathcal{L}_{M}^{(m)} \mathcal{L}_ {\lambda}^{-\frac{1}{2}}\right)^2\right]\leq \frac{\kappa^2}{\lambda}\mathbb{E}\operatorname{tr}\left[\mathcal{L}_ {\lambda}^{-\frac{1}{2}}\mathcal{L}_{M}^{(m)} \mathcal{L}_ {\lambda}^{-\frac{1}{2}}\right]= \frac{\kappa^2}{\lambda}\mathcal{N}_{\mathcal{L}}(\lambda)=:V^2,
\end{align*}
where we used $\|\mathcal{L}_ {\lambda}^{-\frac{1}{2}}\mathcal{L}_{M}^{(m)} \mathcal{L}_ {\lambda}^{-\frac{1}{2}}\|\leq \frac{\kappa^2}{\lambda}$ for the last inequality. 
The claim now follows from Proposition \ref{concentrationineq0}.

\textbf{$E_5)$} Set $w_m:= \mathcal{L}_ {\lambda}^{-\frac{1}{2}}\mathcal{L}_{M}^{(m)}$ . Note that we have
\begin{align*}
\|w_m\|_{HS}&\leq \left\|\mathcal{L}_ {\lambda}^{-\frac{1}{2}}\mathcal{L}_{M}^{(m)} \right\|_{HS} \\
&\leq \left\|\mathcal{L}_ {\lambda}^{-\frac{1}{2}} \left(\sum_{i=1}^p\varphi_m^{(i)}\otimes\varphi_m^{(i)}\right)\right\|_{HS}\\
&\leq \lambda^{-1/2}\sum_{i=1}^p\left\| \varphi_m^{(i)} \right\|_{L^2_{\rho_x}}^2 \leq\frac{\kappa^2}{\sqrt{\lambda}}=:B.
\end{align*}

For the second moment we have,

\begin{align*}
\mathbb{E}\left\|w_m\right\|^2_{HS}\leq \kappa^2 \mathbb{E}\|\mathcal{L}_ {\lambda}^{-\frac{1}{2}}\mathcal{L}_{M}^{(m)} \mathcal{L}_ {\lambda}^{-\frac{1}{2}}\|_{HS} \leq \kappa^2\mathbb{E}\operatorname{tr}\left[\mathcal{L}_ {\lambda}^{-\frac{1}{2}}\mathcal{L}_{M}^{(m)} \mathcal{L}_ {\lambda}^{-\frac{1}{2}}\right]= \kappa^2\mathcal{N}_{\mathcal{L}}(\lambda)=:V^2
\end{align*}

The claim now follows from Proposition \ref{concentrationineq0} together with the fact that the operator norm can be bounded by the Hilbert-Schmidt norm: $\|.\|\leq\|.\|_{HS}$ .

\textbf{$E_6)$} Set $w_m:= \mathcal{L}_{M}^{(m)}$ . Note that we have
\begin{align*}
\|w_m\|_{HS}&\leq \left\|\mathcal{L}_{M}^{(m)} \right\|_{HS} = \left\|\sum_{i=1}^p\varphi_m^{(i)}\otimes\varphi_m^{(i)}\right\|_{HS}\\
&\leq \sum_{i=1}^p\left\| \varphi_m^{(i)} \right\|_{L^2_{\rho_x}}^2 \leq \kappa^2=:B.
\end{align*}

For the second moment we have,

\begin{align*}
\mathbb{E}\left\|w_m\right\|^2_{HS}\leq \kappa^4 =:V^2.
\end{align*}

The claim now follows from Proposition \ref{concentrationineq0}

\textbf{$E_7)$}  Set $w_i:= \xi_i =K_{M,x_i}K_{M,x_i}^*$. Note that 
\begin{align*}
\|w_i\|_{HS}&= \left\|K_{M,x_i}K_{M,x_i}^*\right\|_{HS}\leq \kappa^2 =:B.
\end{align*}

For the second moment we have,

\begin{align*}
\mathbb{E}\left\|w_i\right\|_{HS}^2\leq \kappa^4=:V^2.
\end{align*}

The claim now follows from Proposition \ref{concentrationineq0}.
\end{proof}

\begin{proposition}
\label{concentrationineq1}
Provided Assumption~\ref{ass:input}, the following event holds with probability at least $1-\delta$:
\begin{align*}
E_8
= \left\{
\left\|\Sigma_{M,\lambda}^{-\frac{1}{2}}\widehat{\mathcal{S}}_{M}^{*}\left(\mathbf{v}-\bar{G}_\rho\right)\right\|_{\mathcal{H}_M}
\leq
\left(
\frac{4QZ\kappa}{\sqrt{\lambda}n}
+ \frac{4Q\sqrt{\mathcal{N}_{\mathcal{L}_M}(\lambda)}}{\sqrt{n}}
\right)
\log\!\frac{2}{\delta}
\right\}.
\end{align*}
\end{proposition}

\begin{proof}
We use Proposition~\ref{concentrationineq0} to prove the statement.  
Define
\[
w_i := \Sigma_{M,\lambda}^{-\frac{1}{2}} K_{M,u_i}\bigl(v_i - G_\rho(u_i)\bigr).
\]
Note that $\mathbb{E}[w_i] = 0$ and
\[
\frac{1}{n}\sum_{i=1}^n w_i
= \Sigma_{M,\lambda}^{-\frac{1}{2}}\widehat{\mathcal{S}}_{M}^{*}\left(\mathbf{v}-\bar{G}_\rho\right).
\]
Moreover, by Assumption~\ref{ass:input}, we have
\begin{align*}
\mathbb{E}\!\left[\|w\|_{\mathcal{H}_M}^{\,l}\right]
&= \int_{\mathcal{U}} \int_{\mathcal{V}}
   \|v - G_\rho(u)\|_\mathcal{V}^l \,\rho(dv|u)\,
   \|\Sigma_{M,\lambda}^{-\frac{1}{2}}K_{M,u}\|^l
   \rho_\mathcal{U}(du) \\[4pt]
&\le 2^{l-1} \int_{\mathcal{U}} \int_{\mathcal{V}}
   \!\left(\|v\|_\mathcal{V}^l + Q^l\right)
   \rho(dv|u)\,
   \|\Sigma_{M,\lambda}^{-\frac{1}{2}}K_{M,u}\|^l
   \rho_\mathcal{U}(du) \\[4pt]
&\le 2^{l-1}\!\left(\tfrac{1}{2}l! Z^{\,l-2}Q^2 + Q^l\right)
   \int_{\mathcal{U}}\!\|\Sigma_{M,\lambda}^{-\frac{1}{2}}K_{M,u}\|^l
   \rho_\mathcal{U}(du) \\[4pt]
&\le 2^{l-1}\!\left(\tfrac{1}{2}l! Z^{\,l-2}Q^2 + Q^l\right)
   \sup_{u\in\mathcal{U}}\!
   \|\Sigma_{M,\lambda}^{-\frac{1}{2}}K_{M,u}\|^{l-2}
   \int_{\mathcal{U}}\!\operatorname{tr}\!\bigl(
     \Sigma_{M,\lambda}^{-1}K_{M,u}K_{M,u}^*
   \bigr)\rho_\mathcal{U}(du) \\[4pt]
&\le 2^{l-1}\!\left(\tfrac{1}{2}l! Z^{\,l-2}Q^2 + Q^l\right)
   \left(\tfrac{\kappa}{\sqrt{\lambda}}\right)^{l-2}
   \operatorname{tr}\!\left(
     \Sigma_{M,\lambda}^{-1}\!
     \int_{\mathcal{U}}\!K_{M,u}K_{M,u}^*
     \rho_\mathcal{U}(du)
   \right) \\[4pt]
&\le \tfrac{1}{2}l!\!
   \left(\tfrac{2QZ\kappa}{\sqrt{\lambda}}\right)^{l-2}
   \left(2Q\sqrt{\mathcal{N}_{\mathcal{L}_M}(\lambda)}\right)^2 \\[4pt]
&= \tfrac{1}{2}l!\,B^{l-2}V^{2}.
\end{align*}
Therefore, the statement follows directly from Proposition~\ref{concentrationineq0}.
\end{proof}

\begin{proposition}
\label{concentrationineq2}
Suppose that $\|G_\rho\|_\infty \leq Q$ and that the bound from Proposition~\ref{ineq5} holds:
\[
\|F^*_\lambda\|_\infty \leq C_{\kappa,R,D}\,\lambda^{-(\frac{1}{2}-r)^+},
\quad\text{where } C_{\kappa,R,D} = 2\kappa^{2r+1} R D.
\]
Then, the following event holds with probability at least $1-\delta$:
\begin{align*}
E_9
= \left\{
\left|
\frac{1}{n}\bigl\|\bar{G}_\rho - \widehat{\mathcal{S}}_M F_\lambda^*\bigr\|_2^2
- \bigl\|G_\rho - \mathcal{S}_M F_\lambda^*\bigr\|_{L^2(\rho_\mathcal{U})}^2
\right|
\le
2\!\left(\frac{B_\lambda}{n} + \frac{V_\lambda}{\sqrt{n}}\right)
\log\!\frac{2}{\delta}
\right\},
\end{align*}
where
\[
B_\lambda := 4\!\left(Q^2 + C_{\kappa,R,D}^2\,\lambda^{-2(\frac{1}{2}-r)^+}\right),
\qquad
V_\lambda := \sqrt{2}\!\left(Q + C_{\kappa,R,D}\,\lambda^{-(\frac{1}{2}-r)^+}\right)
\bigl\|G_\rho - \mathcal{S}_M F_\lambda^*\bigr\|_{L^2(\rho_\mathcal{U})}.
\]
\end{proposition}

\begin{proof}
We apply Proposition~\ref{concentrationineq0}.  
Define
\[
w_i := \bigl\|G_\rho(u_i) - F_\lambda^*(u_i)\bigr\|_\mathcal{V}^2.
\]
Then $\mathbb{E}[w_i] = \|G_\rho - \mathcal{S}_M F_\lambda^*\|_{L^2(\rho_\mathcal{U})}^2$, and therefore
\[
\left|
\frac{1}{n}\sum_{i=1}^n w_i - \mathbb{E}[w_1]
\right|
= 
\left|
\frac{1}{n}\bigl\|\bar{G}_\rho - \widehat{\mathcal{S}}_M F_\lambda^*\bigr\|_2^2
- \bigl\|G_\rho - \mathcal{S}_M F_\lambda^*\bigr\|_{L^2(\rho_\mathcal{U})}^2
\right|.
\]

It remains to bound $|w_i|$ and $\mathbb{E}[w_1^2]$.  
Using $\|G_\rho\|_\infty := \sup_{u\in\mathcal{U}} \|G_\rho(u)\|_\mathcal{V} \le Q$ and Proposition~\ref{ineq5}, we obtain
\begin{align*}
|w_i|
&\le 2\!\left(Q^2 + C_{\kappa,R,D}^2\,\lambda^{-2(\frac{1}{2}-r)^+}\right),
\end{align*}
and further,
\begin{align*}
\mathbb{E}\bigl[w_1^2\bigr]
&\le 2\!\left(Q^2 + C_{\kappa,R,D}^2\,\lambda^{-2(\frac{1}{2}-r)^+}\right)
   \mathbb{E}[w_1] \\[2pt]
&= 2\!\left(Q^2 + C_{\kappa,R,D}^2\,\lambda^{-2(\frac{1}{2}-r)^+}\right)
   \bigl\|G_\rho - \mathcal{S}_M F_\lambda^*\bigr\|_{L^2(\rho_\mathcal{U})}^2.
\end{align*}
Hence, the claim follows directly from Proposition~\ref{concentrationineq0}.
\end{proof}